\documentclass{article}

\usepackage[final,nonatbib]{neurips_2023}

\usepackage[numbers]{natbib}
\usepackage[utf8]{inputenc} % allow utf-8 input
\usepackage[T1]{fontenc}    % use 8-bit T1 fonts
\usepackage{hyperref}       % hyperlinks
\usepackage{url}            % simple URL typesetting
\usepackage{booktabs}       % professional-quality tables
\usepackage{amsfonts}       % blackboard math symbols
\usepackage{nicefrac}       % compact symbols for 1/2, etc.
\usepackage{microtype}      % microtypography
\usepackage{xcolor}         % colors
\usepackage{ifsym}
\newcommand{\ind}{\perp\!\!\!\perp} 
\newcommand{\starto}{*\!\!\to}
\def\XXint#1#2#3{{\setbox0=\hbox{$#1{#2#3}{\int}$ }
\vcenter{\hbox{$#2#3$ }}\kern-.6\wd0}}

\usepackage[ruled,linesnumbered]{algorithm2e}
\usepackage{setspace}
\usepackage[outline]{contour}

\usepackage{amsmath,amssymb}
\usepackage{MnSymbol}
\usepackage{mathtools}
\usepackage{amsthm}
\usepackage{multirow}

\theoremstyle{plain}
\newtheorem{theorem}{Theorem}[section]
\newtheorem{proposition}[theorem]{Proposition}
\newtheorem{lemma}[theorem]{Lemma}

\theoremstyle{definition}
\newtheorem{definition}[theorem]{Definition}
\newtheorem{assumption}[theorem]{Assumption}
\theoremstyle{remark}
\newtheorem{remark}[theorem]{Remark}
\newtheorem{example}[theorem]{\textbf{\emph{Example}}}

\newcommand{\doubleleftarrows}{\small \leftarrow\!\cdots\!\leftarrow}
\newcommand{\doubleto}{\small \to\!\cdots\!\to}
  
\usepackage[inline, shortlabels]{enumitem}
\usepackage{footnote}
\makesavenoteenv{algorithm}

\usepackage[toc,page,header]{appendix}
\usepackage{minitoc}

\usepackage{graphicx}

\title{Causal Discovery from Subsampled Time Series \\ with Proxy Variables}

\author{%
  Mingzhou Liu$^{1,2}$\\
  \And
  Xinwei Sun$^3$\thanks{Correspondence to sunxinwei@fudan.edu.cn} \\
  \And
  Lingjing Hu$^4$ \\
  \And
  Yizhou Wang$^{1,2,5}$ \\
}
\date{\newline\newline
    $^1$ School of Computer Science, Peking University \\
    $^2$ Center on Frontiers of Computing Studies (CFCS), Peking University\\ 
    $^3$ School of Data Science, Fudan University \\ 
    $^4$ Yanjing Medical College, Capital Medical University \\ 
    $^5$ Institute for Artificial Intelligence, Peking University\\
} 
\let\oldmaketitle\maketitle
\renewcommand{\maketitle}{\oldmaketitle\setcounter{footnote}{0}}

\begin{document}
\maketitle

% add the following to make toc work
\addtocontents{toc}{\protect\setcounter{tocdepth}{0}}

\begin{abstract}
    Inferring causal structures from time series data is the central interest of many scientific inquiries. A major barrier to such inference is the problem of subsampling, \emph{i.e.}, the frequency of measurement is much lower than that of causal influence. To overcome this problem, numerous methods have been proposed, yet either was limited to the linear case or failed to achieve identifiability. In this paper, we propose a constraint-based algorithm that can identify the entire causal structure from subsampled time series, without any parametric constraint. Our observation is that the challenge of subsampling arises mainly from hidden variables at the unobserved time steps. Meanwhile, every hidden variable has an observed proxy, which is essentially itself at some observable time in the future, benefiting from the temporal structure. Based on these, we can leverage the proxies to remove the bias induced by the hidden variables and hence achieve identifiability. Following this intuition, we propose a proxy-based causal discovery algorithm. Our algorithm is nonparametric and can achieve full causal identification. Theoretical advantages are reflected in synthetic and real-world experiments. Our code is available at \url{https://github.com/lmz123321/proxy_causal_discovery}.
\end{abstract}
\section{Introduction}

Temporal systems are the primary subject of causal modeling in many sciences, such as pathology, neuroscience, and economics. For these systems, a common issue is the difficulty of collecting sufficiently refined timescale data. That is, we can only observe a \emph{subsampled} version of the true causal interaction. This can cause serious problems for causal identification, since full observation (\emph{i.e.}, causal sufficiency \cite{pearl2009causality}) is believed to be an important condition for causal identification, and the violation of which (\emph{e.g.}, the existence of hidden variables) can induce strong bias \cite{suppes1980causal}.
\vspace{+0.15cm}
\begin{example}
    Take the pathology study of Alzheimer's disease (AD) as an example. In AD, patients suffer from memory loss due to the atrophy of memory-related brain regions such as the Hippocampus \cite{west1994differences}. To monitor such atrophy, the standard protocol is to perform a MRI examination on the brain every six months \cite{petersen2010alzheimer}. However, many studies have shown that the disease can progress much more rapidly than this \cite{doody2001method,thalhauser2012alzheimer}. For this reason, when recovering the causal interactions between brain regions, those at the unobserved time steps constitute hidden variables and induce spurious edges in the causal graph, as shown in Figure~\ref{fig.ad_8nodes} (b).
\end{example}

\begin{figure}[tp]
    \centering
    \includegraphics[width=.9\textwidth]{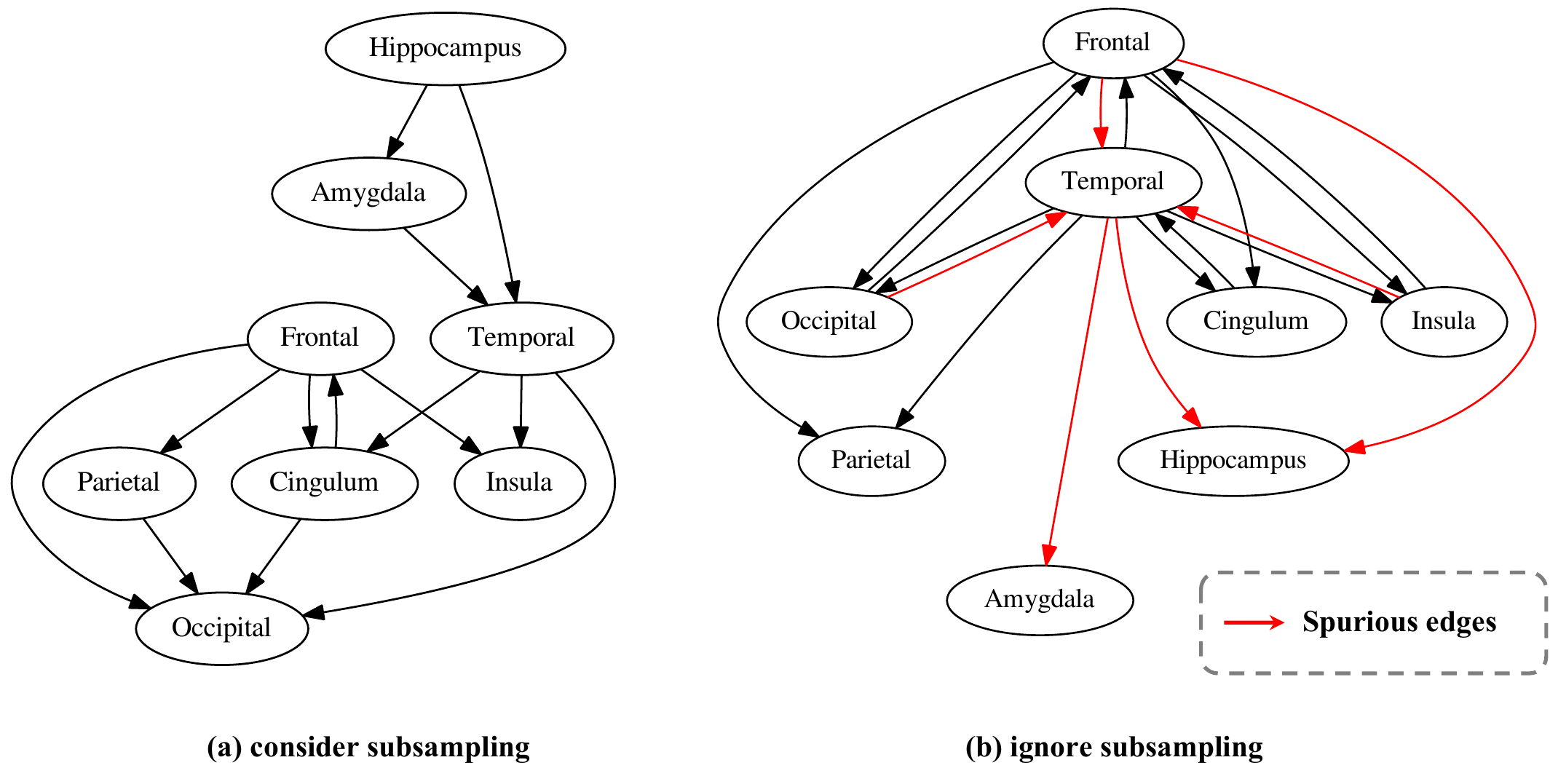}
    \caption{Recovered causal pathways in Alzheimer's disease, with (a) our method that explicitly models the subsampling and (b) Dynotears \cite{pamfil2020dynotears} that neglects subsampling. 
    Spurious edges that contradict clinical studies \cite{young2018uncovering,vogel2021four} are marked red. Please refer to Section~\ref{sec.exp ad} and Figure~\ref{fig.ad_90nodes} for details.}
    \label{fig.ad_8nodes}
\end{figure}

To resolve this problem, many attempts have been made to recover the correct causal relations from subsampled data. However, none of them achieved general identifiability results. For example, in \citep{gong2015discovering,gong2017causal,tank2019identifiability}, identifiability was achieved only for linear data. As for nonlinear data, only a small part of the causal information (\emph{i.e.}, an equivalence class) could be identified \cite{danks2013learning,plis2015rate,hyttinen2016causal,malinsky2018causal}. Indeed, encountering such difficulties is no coincidence, considering the fact that there are many hidden variables in the system. Particularly, the hidden mediators prevent us from distinguishing the direct causation from possible mediation and hence only allow us to recover some ancestral information.

In this paper, we propose a constraint-based algorithm that can identify the entire causal structure from subsampled data, without any parametric assumption. Our departure here is that in time series, every hidden variable has an observed descendent, which is essentially itself at some observable time in the future\footnote{For example, the hidden variable $X(t_u)$ at the unobserved time step $t_u$ with $X(t_u)\to X(t_u+1) \to X(t_u+2) \doubleto X(t_o)$ has an observable descendent $X(t_o)$ at some observable time $t_o$ in the future.}. Therefore, we can use the observed descendent as its \emph{proxy variable} \cite{kuroki2014measurement} and adjust for the bias induced by it. For this purpose, we first represent the ancestral information, which is composed of both direct causation and hidden mediation, with the Maximal Ancestral Graph (MAG) \cite{spirtes1999computation}. Then, to distinguish causation from mediation, we use the observed descendants of the hidden mediators as their proxy variables and adjust for their induced mediation bias. We show that our strategy can achieve a complete identification of the causal relations entailed in the time series. It is essentially nonparametric and only relies on common continuity and differentiability of the structural equation to conduct valid proxy-based adjustment \cite{liu2023causal}. 

Back to the AD example, as shown in Figure~\ref{fig.ad_8nodes}, our method can recover causal pathways that align well with existing clinical studies \cite{young2018uncovering,vogel2021four}, while compared baselines fail to do so.

Our contributions are summarized below:
\begin{enumerate}[leftmargin=25pt,itemsep=0.15mm]
    \item We establish nonparametric causal identification in subsampled time series.
    \item We propose an algorithm to practically recover the causal graph.
    \item We achieve more accurate causal identification than others on synthetic and real data.
\end{enumerate}

The rest of the paper is organized as follows. First, in Section~\ref{sec.preliminary}, we introduce the problem setting, basic assumptions, and related literature needed to understand our work. Second, in Section~\ref{sec.identify}, we establish theories on structural identifiability. Then, based on these theories, in Section~\ref{sec.alg}, we introduce the proposed causal discovery algorithm and verify its effectiveness in Section~\ref{sec.exp}. Finally, we conclude the paper and discuss future works in Section~\ref{sec.conclusion}.

\section{Preliminary}
\label{sec.preliminary}

We start by introducing the causal framework and related literature needed to understand our work. 

\noindent\textbf{Time series model.} Let $\mathbf{X}(t):=\left[X_1(t),...,X_d(t)\right]$ be a multivariate time series with $d$ variables defined at discrete time steps $t=1,...,T$. We assume the data is generated by a first-order structural vector autoregression (SVAR) process \cite{malinsky2018causal}:
\begin{equation}
    X_i(t) = f_i(\mathbf{PA}_i(t-1),N_i),
\label{eq.svar model}
\end{equation}
where $f_i$ is the structural function, $\mathbf{PA}_i$ is the parent of $X_i$, and $N_i$ is the exogenous noise.

Implicit in (\ref{eq.svar model}) is the assumptions that cause precedes effect, and that causation is invariant to time, \emph{i.e.}, the structural function $f_i$ and the causal parents $\mathbf{PA}_i$ keeps unchanged across time steps. These assumptions carry the fundamental beliefs of temporal precedence and stability on causality \cite{pearl2009causality}, therefore are widely adopted by existing works \cite{gong2015discovering,gong2017causal,tank2019identifiability,plis2015rate,malinsky2018causal}.

\noindent\textbf{Subsampling.} The subsampling problem means that model (\ref{eq.svar model}) can be only observed every $k$ steps \cite{gong2015discovering,tank2019identifiability,malinsky2018causal}. That is, we can only observe $\mathbf{X}(1),\mathbf{X}(k+1),...,\mathbf{X}( \lfloor\frac{T}{k}\rfloor k+1)$.

\begin{figure}[tp]
    \centering
    \includegraphics[width=\textwidth]{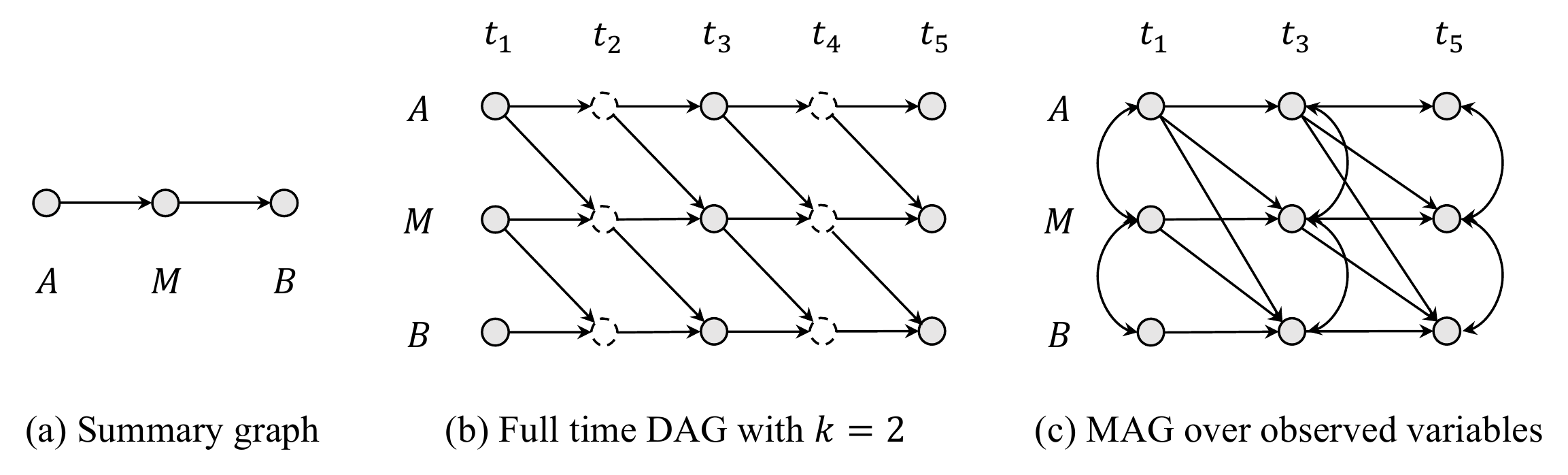}
    \caption{Graph terminologies: (a) Summary graph. (b) Full time DAG with subsampling factor $k=2$. Variables at observed time steps $t_1,t_3,t_5$ are marked gray, while those at unobserved time steps $t_2,t_4$ are dashed. (c) MAG over observed variables. The edges $A(t_1)\to M(t_3)$ and $A(t_3)\leftrightarrow M(t_3)$ are (\emph{resp.}) induced by the inducing paths $A(t_1)\to A(t_2)\to M(t_3)$ and $A(t_3) \leftarrow A(t_2) \to M(t_3)$ in the full time DAG. Please refer to the appendix for detailed explanations.}
    \label{fig.illustration graphs}
\end{figure}

\noindent\textbf{Graph terminologies.} Causal relations entailed in model (\ref{eq.svar model}) can be represented by causal graphs (Figure~\ref{fig.illustration graphs}). We first introduce the \emph{full time Directed Acyclic Graph (DAG)} \cite{gong2023causal}, which provides a complete description of the dynamics in the system.

\begin{definition}[Full time DAG]
Let $G:=(\mathbf{V},\mathbf{E})$ be the associated full time DAG of model (\ref{eq.svar model}). The vertex set $\mathbf{V}:=\{\mathbf{X}(t)\}_{t=1}^T$, the edge set $\mathbf{E}$ contains $X_i(t-1) \to X_j(t)$ iff $X_i\in \mathbf{PA}_j$.
\end{definition}

We assume the full time DAG is Markovian and faithful to the joint distribution $\mathbb{P}(\mathbf{X}(1),...,\mathbf{X}(T))$.

\begin{assumption}[Markovian and faithfulness]
For disjoint vertex sets $\mathbf{A},\mathbf{B},\mathbf{Z} \subseteq \mathbf{V}$, $\mathbf{A} \ind \mathbf{B} | \mathbf{Z} \Leftrightarrow \mathbf{A} \ind_G \mathbf{B} | \mathbf{Z}$, where$\ind_G$ denotes \emph{d}-separation in $G$.
\label{asm.mark and faith}
\end{assumption}

In practice, it is often sufficient to know the causal relations between time series as a whole, without knowing precisely the relations between time instants \cite{assaad2022discovery}. In this regard, we can summarize the causal relations with the \emph{summary graph} \cite{gong2023causal}.

\begin{definition}[Summary graph]
Let $G:=(\mathbf{V},\mathbf{E})$ be the associated summary graph of model (\ref{eq.svar model}). The vertex set $\mathbf{V}:=\mathbf{X}$, the edge set $\mathbf{E}$ contains $X_i \to X_j (i \neq j)$ iff $X_i \in \mathbf{PA}_j$.
\end{definition}

In this paper, our goal is to recover the summary graph, given data observed at $t=1,k+1,...,\lfloor \frac{T}{k}\rfloor k+1$. 

For this purpose, we need a structure to represent the (marginal) causal relations between observed variables and hence link the observation distribution to the summary graph. Here, we use the \emph{Maximal Ancestral Graph (MAG)} \cite{zhang2008completeness}, since it can represent casual relations when unobserved variables exist.

Specifically, for the variable set $\mathbf{V}:=\{\mathbf{X}(t)\}_{t=1}^T$, let $\mathbf{O}:=\{\mathbf{X}(1),...,\mathbf{X}(\lfloor \frac{T}{k} \rfloor k+1)\}$ be the observed subset and $\mathbf{L}:=\mathbf{V}\backslash \mathbf{O}$ be the unobserved subset. Let $\mathbf{An}(A)$ be the ancestor set of $A$. Then, given any full time DAG $G$ over $\mathbf{V}$, the corresponding MAG $M_G$ over $\mathbf{O}$ is defined as follows:

\begin{definition}[MAG]
In the MAG $M_G$, two vertices $A,B \in \mathbf{O}$ are adjacent iff there is an inducing path. \footnote{A path $p$ between $A,B \in \mathbf{O}$ is an inducing path relative to $\mathbf{L}$ if every non-endpoint vertex on $p$ is either in $\mathbf{L}$ or a collider, and every collider on $p$ is an ancestor of either $A$ or $B$.} relative to $\mathbf{L}$ between them in $G$. Edge orientation is:

\begin{enumerate}[leftmargin=25pt]
    \item $A \to B$ in $M_G$ if $A \in \mathbf{An}(B)$ in $G$; 
    \item $A \leftarrow B$ in $M_G$ if $B \in \mathbf{An}(A)$ in $G$;
    \item $A \leftrightarrow B$ in $M_G$ if $A \not\in \mathbf{An}(B)$ and $B \not\in \mathbf{An}(A)$ in $G$.
\end{enumerate}

% \quad\,
% \begin{enumerate*}[series=tobecont,itemjoin=\quad]
%     \item $A \to B$ in $M_G$ if $A \in \mathbf{An}(B)$ in $G$; \item $A \leftarrow B$ in $M_G$ if $B \in \mathbf{An}(A)$ in $G$;
% \end{enumerate*}
% \begin{enumerate}[resume=tobecont, ,topsep=0pt, partopsep=0pt, leftmargin=27pt]
%     \item $A \leftrightarrow B$ in $M_G$ if $A \not\in \mathbf{An}(B)$ and $B \not\in \mathbf{An}(A)$ in $G$.
% \end{enumerate}

\label{def.mag}
\end{definition}

\begin{remark}
    In the MAG, the directed edge $A \to B$ means $A$ is the ancestor of $B$. The bidirected edge $A \leftrightarrow B$ means there is an unobserved confounder $U$ between $A$ and $B$. 
\label{remark.mag}
\end{remark}

\noindent\textbf{Proximal causal discovery.} Recent works \cite{miao2018identifying,liu2023causal} found that one could use the descendant (\emph{i.e.}, proxy) of an unobserved variable $M$ to differentiate direct causation $A\to B$ from hidden mediation $A\to M \to B$. Specifically, suppose that:

\begin{assumption}[TV Lipschitzness]
    For each causal pair $A\to B$, the map $a\mapsto \mathbb{P}(B|A=a)$ is Lipschitz with respect to the Total Variation (TV) distance, that is, there exists a constant $L$ such that:
    \begin{equation*}
        \mathrm{TV}(\mathbb{P}(B|A=a),\mathbb{P}(B|A=a^\prime)) \leq L|a-a^\prime|.
    \end{equation*}
    \label{asm.tv-smooth}
\end{assumption}
\begin{assumption}[Completeness]
    For each causal pair $A\to B$, the conditional distribution $\mathbb{P}(B|A)$ is complete, that is, for any function $g$, we have:
    \begin{equation*}
        E\{g(b)|a\}=0 \,\, \text{almost surely} \,\, \text{if and only if} \,\, g(b)=0 \,\, \text{almost surely}.
    \end{equation*}
    \label{asm.comp}
\end{assumption}

\begin{remark}
    Roughly, Assumption~\ref{asm.tv-smooth} requires the continuity of the structural function $f_i$ and Assumption~\ref{asm.comp} requires that the characterization function of the exogenous noise $N_i$ is non-zero. Please refer to \citep{liu2023causal} for a detailed discussion.
\end{remark}

Under the above assumptions, \citep{liu2023causal} showed that:

\begin{theorem}[\citep{liu2023causal}]
Suppose that $A$ and $B$ are mediated by an unobserved variable $M$, and that $M$ has a proxy variable $M^\prime$ satisfying $A\ind_G M^\prime|M$. Then, we can identify whether $A\to B$ by testing whether $pr(b|a)\sim pr(m^\prime|a)$ has a linear relation. If the linear relation exists, then $A\not\to B$, otherwise, we have$A\to B$.
\label{thm.proxy}
\end{theorem}

\begin{remark}
The linear relation can be tested by checking the least square residual of regressing $pr(b|a)$ on $pr(m^\prime|a)$. Please refer to \citep{liu2023causal} for details.
\end{remark}

To ensure the existence of the proxy variable, we require the following self causation assumption, which states that each variable $X_i(t)$ is influenced by its past $X_i(t-1)$.

\begin{assumption}[Self causation]
In model (\ref{eq.svar model}), $X_i \in \mathbf{PA}_i$ for each $i=1,...,d$.
\label{asm.self cause}
\end{assumption}

\begin{remark}
    In time series statistics, this is also known as the autocorrelation assumption \cite{bartlett1946theoretical,lomnicki1957estimation}. 
\end{remark}

\section{Structural identifiability}
\label{sec.identify}

In this section, we establish the identifiability of the summary graph. For this purpose, we connect the observational distribution to the summary graph with the bridge of the MAG. 

Specifically, our analysis consists of three progressive results: Proposition~\ref{prop.iden mag}, Proposition~\ref{prop.sdag and mag}, and Theorem~\ref{thm.iden sdag}. First, in Proposition~\ref{prop.iden mag}, we prove that the MAG can be identified from observational data. Then, in Proposition~\ref{prop.sdag and mag}, we show that the identified MAG \emph{almost uniquely} reflects the structure of the summary graph, in the sense that only edges connecting vertices and their ancestors cannot be determined due to hidden mediators. Finally, in Theorem~\ref{thm.iden sdag}, we identify these underdetermined edges with the proxies of the hidden mediators and therefore identify the whole summary graph.

\newpage
Next, we first introduce Proposition~\ref{prop.iden mag}, which shows that the MAG is identifiable.

\begin{proposition}[Identifiability of the MAG]
Assuming model (\ref{eq.svar model}) and Assumption~\ref{asm.mark and faith}, then the MAG over the observed variable set $\mathbf{O}$ is identifiable, i.e., its skeleton and edge orientations can be uniquely derived from the joint distribution $\mathbb{P}(\mathbf{O})$.
\label{prop.iden mag}
\end{proposition}

Since the identifiability of the MAG's skeleton under the faithfulness assumption is a well-known result \cite{spirtes1999computation}, we focus on explaining the identification of edge orientations. Specifically, we will divide the edges into two classes: the instantaneous edges and the lagged edges, and discuss their orientations respectively. 

The \emph{instantaneous edge}, \emph{e.g.}, $A(t_3)\leftrightarrow M(t_3)$, is an edge that connects two vertices at the same time. Since we assume cause must precede effect, the instantaneous edge does not represent causation but latent confounding, therefore is bidirected according to Definition~\ref{def.mag}. For example, in Figure~\ref{fig.illustration graphs} (c), the instantaneous edge $A(t_3)\leftrightarrow M(t_3)$ represents the latent confounding $A(t_3) \leftarrow A(t_2) \to M(t_3)$ between $A(t_3)$ and $M(t_3)$. On the other hand, the \emph{lagged edge}, \emph{e.g.}, $A(t_1)\to M(t_3)$, is an edge that connects two vertices at different time. The lagged edge represents ancestral information and hence is directed ($\to$) from the past to the future. 

To connect the identified MAG to the summary graph, we first define the following graph structures, examples of which are shown in Figure~\ref{fig.graph structs}:

\begin{definition}[Graph structures]
In the summary graph,
\begin{enumerate}[leftmargin=25pt,itemsep=0.15mm]
    \item A \emph{directed path} $p_{AB}$ from $A$ to $B$ with length $l$ is a sequence of distinct vertices $A,V_1,...,V_l,B$ where each vertex points to its successor. Two directed paths $p_{A B_1}, p_{A B_2}$ are called \emph{disjoint} if they do not share any non-startpoint vertex.
    \item A \emph{confounding structure} $c_{AB}$ between $A$ and $B$ with lengths $(r,q)$ consists of a vertex $U$, a directed path $p_{UA}$ from $U$ to $A$ with length $r$, and a directed path $p_{UB}$ from $U$ to $B$ with length $q$, where $p_{UA}$ and $p_{UB}$ are disjoint.
\end{enumerate}
\label{def.conf struct}
\end{definition}

\begin{figure}[tp]
    \centering
    \includegraphics[width=.85\textwidth]{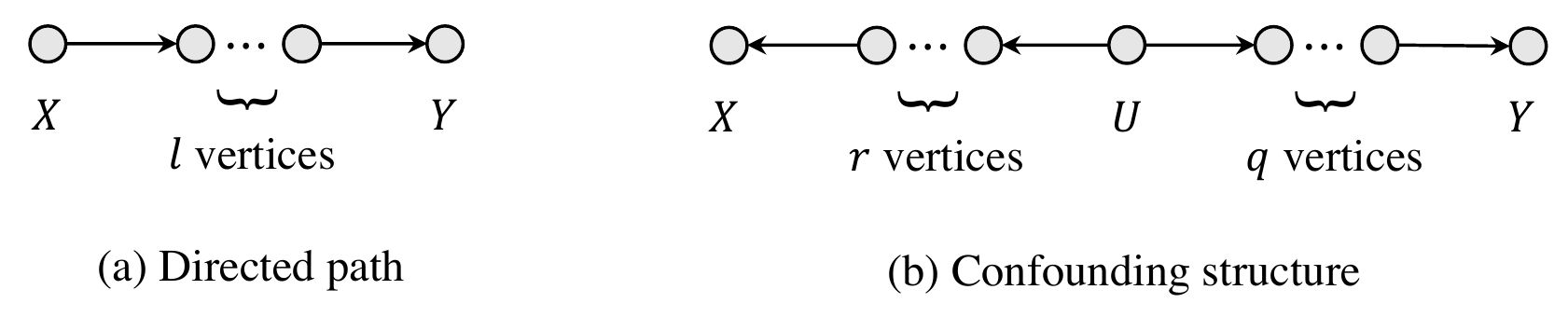}
    \caption{Illustration of Definition~\ref{def.conf struct}. (a) A directed path from $A$ to $B$ with length $l$. (b) A confounding structure between $A$ and $B$ with lengths $(r,q)$.}
    \label{fig.graph structs}
\end{figure}

In the following, we explain the relationship between the identified MAG and the summary graph. In particular, we will discuss what the directed and bidirected edges in the MAG (\emph{resp.}) imply about the summary graph. 

According to Definition~\ref{def.mag}, the \emph{directed edge} $A(t)\to B(t+k)$ in the MAG means $A$ is the ancestor of $B$. Therefore, in the summary graph, there is either $A\to B$ or a directed path from $A$ to $B$. On the other hand, the \emph{bidirected edge} $A(t+k)\leftrightarrow B(t+k)$ in the MAG means there is a latent confounder between them, which can be either $A(t+1)$, \emph{i.e.}, $A(t+k) \doubleleftarrows A(t+1) \doubleto B(t+k)$, or a third variable $U(t+1)$, \emph{i.e.}, $A(t+k) \doubleleftarrows U(t+1) \doubleto B(t+k)$. Hence, in the summary graph, there is either a directed path from $A$ to $B$ or a confounding structure between them. 

Summarizing the above observations, we have the following proposition:

\begin{proposition}[MAG to summary graph]
If there are $A(t) \to B(t+k)$ and $A(t+k) \leftrightarrow B(t+k)$ in the MAG, then, in the summary graph, there is either 

% \begin{enumerate}[leftmargin=25pt]
%     \item $A \to B$,
%     \item a directed path from $A$ to $B$ with length $l\leq k-2$, or
%     \item a directed path from $A$ to $B$ with length $l=k-1$ and a confounding structure between them with lengths $(r\leq k-2,q\leq k-2)$.
% \end{enumerate}

\quad\,
\begin{enumerate*}[series=tobecont,itemjoin=\quad]
    \item $A \to B$, \item a directed path from $A$ to $B$ with length $l\leq k-2$, or
\end{enumerate*}
\begin{enumerate}[resume=tobecont, ,topsep=0pt, partopsep=0pt, leftmargin=25pt]
    \item a directed path from $A$ to $B$ with length $l=k-1$ and a confounding structure between them with lengths $(r\leq k-2,q\leq k-2)$.
\end{enumerate}
\label{prop.sdag and mag}
\end{proposition}
\vspace{+1em}

\begin{remark}
In Proposition~\ref{prop.sdag and mag}, the maximum length of the directed path is $k-1$, this is because any path longer than this can not be an inducing path and therefore can not induce $A(t)\to B(t+k)$ in the MAG. Please refer to the appendix for details.
\end{remark}

Proposition~\ref{prop.sdag and mag} provides a necessary condition for having $A\to B$ in the summary graph, \emph{i.e.}, there are both $A(t) \to B(t+k)$ and $A(t+k) \leftrightarrow B(t+k)$ in the MAG. It also inspires us that, to make the condition sufficient, we need to distinguish the direct effect ($A\to B$) from the indirect one ($A \doubleto B$).

For this purpose, we propose to use the proximal causal discovery method. Before preceding any technical detail, we first brief our idea below. The indirect effect, unlike the direct one, relies on mediation and therefore can be \emph{d}-separated by the (unobserved) mediator. To test the \emph{d}-separation, we can use the observed proxy of the mediator, thanks to the self causation assumption in Assumption~\ref{asm.self cause}. For example, in Figure~\ref{fig:sdag proxy} (b), the indirect effect $A(t_1)\to M(t_2) \to B(t_3)$ from $A$ to $B$ can be \emph{d}-separated by the mediator $M(t_2)$, who has an observed proxy $M(t_3)$.

For the general case, the separation set is a bit more complicated to ensure all paths except the one representing direct effect are \emph{d}-separated. Specifically, for two vertices $A(t)$ and $B(t+k)$, the separation set is the union of two sets $\mathbf{M}(t+1) \cup \mathbf{S}(t)$. The set $\mathbf{M}(t+1)$ contains possible mediators, namely $A(t+1)$ and any vertex $M_i(t+1)$ ($M_i \neq B$) such that $A(t)\to M_i(t+k)$ in the MAG and $M_i$ is not $B$'s descendant\footnote{This can be justified from the MAG. Specifically, if $M_i$ is the descendant of $B$, then there is $B(t) \to V_1(t+k), V_1(t) \to V_2(t+k), ..., V_l(t)\to M_i(t+k)$ in the MAG.}. The set $\mathbf{S}(t)$ is used to \emph{d}-separate possible back-door paths between $A(t)$ and $B(t+k)$, it contains any vertex $S_i(t)$ ($S_i \neq A$) such that $S_i(t) \to B(t+k)$ or $S_i(t) \to M_j(t+k)$ for some $M_j \in \mathbf{M}$ in the MAG. 

Equipped with the separation set, we then have the following identifiability result:

\begin{figure}[tp]
    \centering
    \includegraphics[width=\textwidth]{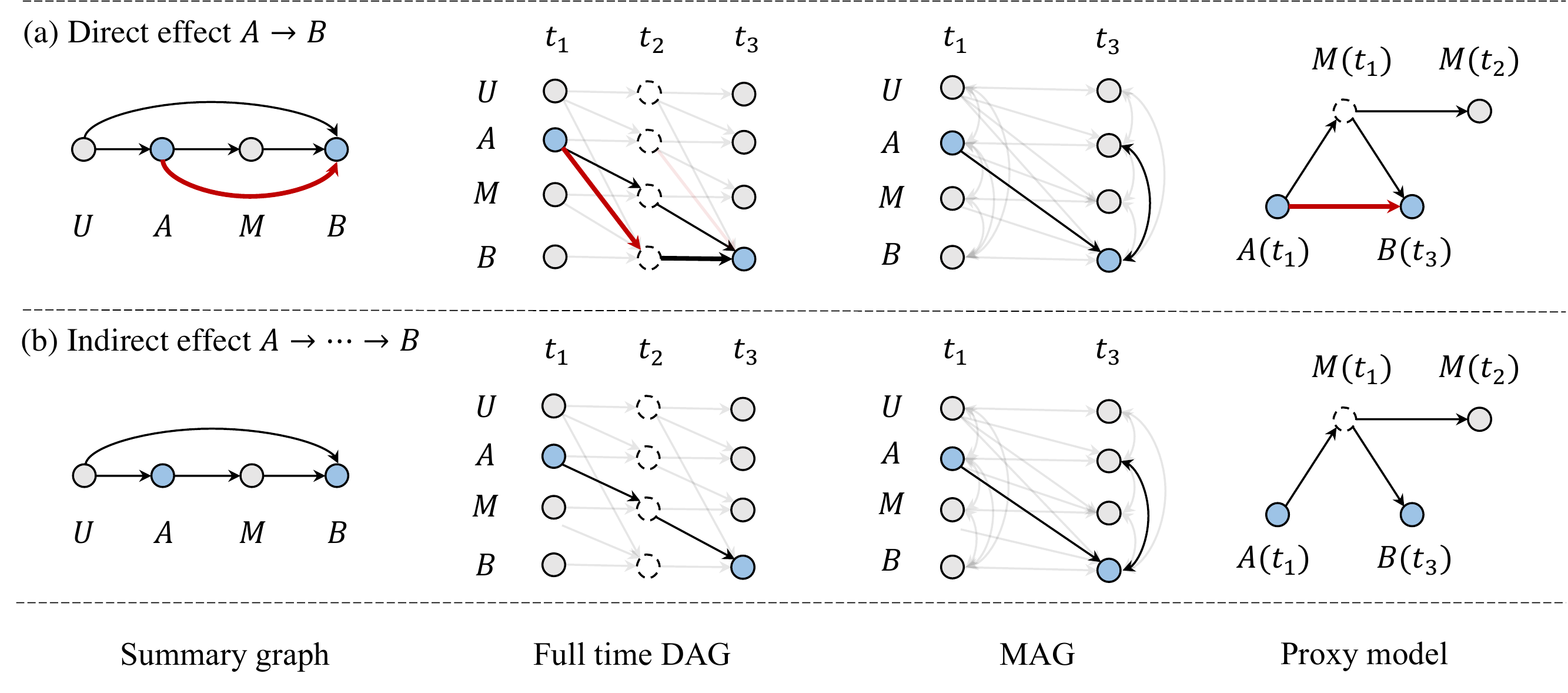}
    \caption{Distinguishing the direct effect ($A\to B$) from the indirect one ($A\doubleto B$) with proxy variables. Note that though (a) and (b) have the same MAG, only the indirect effect $A(t_1)\to M(t_2) \to B(t_3)$ in (b) can be \emph{d}-separated by $M(t_2)$. To test the \emph{d}-separation, we can use $M(t_3)$ as the proxy variable.}
    \label{fig:sdag proxy}
\end{figure}

\begin{theorem}[Identifiability of the summary graph]
Assuming model (\ref{eq.svar model}), Assumption~\ref{asm.mark and faith}, and Assumptions~\ref{asm.tv-smooth}, \ref{asm.comp}, \ref{asm.self cause}, then the summary graph is identifiable. Specifically,
\begin{enumerate}[leftmargin=25pt,itemsep=0.15mm]
    \item There is $A\to B$ in the summary graph iff there are $A(t)\to B(t+k), A(t+k)\leftrightarrow B(t+k)$ in the MAG, and the set $\mathbf{M}(t+1)\cup \mathbf{S}(t)$ is not sufficient to d-separate $A(t),B(t+k)$ in the full time DAG. 
    \item The condition ``the set $\mathbf{M}(t+1)\cup \mathbf{S}(t)$ is not sufficient to d-separate $A(t),B(t+k)$ in the full time DAG'' can be tested by the proxy variable $\mathbf{M}(t+k)$ of the unobserved set $\mathbf{M}(t+1)$.
\end{enumerate}
\label{thm.iden sdag}
\end{theorem}

Theorem~\ref{thm.iden sdag} provides a necessary and sufficient condition for having $A\to B$ in the summary graph. This condition is testable with the proxy $\mathbf{M}(t+k)$ of the unobserved set $\mathbf{M}(t+1)$ ($\mathbf{S}(t)$ is observable). The validity of such proxying, \emph{i.e.}, $A(t) \ind_G \mathbf{M}(t+k) | \mathbf{M}(t+1), \mathbf{S}(t)$ is proved in the appendix.

\section{Discovery algorithm}
\label{sec.alg}

Equipped with the identifiability results in Proposition~\ref{prop.iden mag} and Theorem~\ref{thm.iden sdag}, we can practically discover the summary graph with Algorithm~\ref{alg:iden sdag}. Specifically, given the observed data points, Algorithm~\ref{alg:iden sdag} first recovers the MAG according to Proposition~\ref{prop.iden mag}. Then, for each pair of vertices in the summary graph, it first uses the necessary condition of ``having $A(t)\to B(t+k)$ and $A(t+k)\leftrightarrow B(t+k)$ in the MAG'' to construct underdetermined edges. Finally, among these underdetermined edges, it uses the proxy variable approach to identify direct causation and remove the indirect one, according to Theorem~\ref{thm.iden sdag}.

Below, we first introduce the \emph{Partially Determined Directed Acyclic Graph (PD-DAG)}, which will be used to represent the intermediate result of the algorithm. 

\begin{definition}[PD-DAG]
A PD-DAG is a DAG with two kinds of directed edges: solid ($\to$) and dashed ($\dashrightarrow$). A solid edge represents a determined causal relation, while a dashed one means the causal relation is still underdetermined. We use a meta-symbol, asterisk ($\starto$), to denote any of the two edges. In a PD-DAG $G=(\mathbf{V},\mathbf{E})$, the vertex $A$ points to the vertex $B$ if the edge $A \starto B \in \mathbf{E}$. The definition of the directed path and confounding structure are the same as in the summary graph.
\end{definition}

\SetKwComment{Comment}{/* }{ */}

\begin{algorithm}[h!]
\caption{Discover the summary graph}
\setstretch{1.2}
\label{alg:iden sdag}
\KwData{Data observed at $t=1,k+1,...,\lfloor\frac{T}{k}\rfloor k+1$}
\KwResult{The summary graph}
Construct the skeleton of the MAG with \cite{spirtes2000causation}, orient edges according to Proposition~\ref{prop.iden mag};

For every vertex pair $A,B$, set $A\starto B$ if there are $A(t)\to B(t+k)$ and $A(t+k)\leftrightarrow B(t+k)$ in the MAG. Called the resulted PD-DAG $G_0$ \Comment*[r]{Proposition~\ref{prop.sdag and mag}}

Set $G=G_0$ and iteratively execute:
\begin{enumerate}[label=(\alph*)]
    \item For every $A\starto B$, check
    \begin{enumerate}[label=(\roman*)]
        \item a directed path from $A$ to $B$ with length $l\leq k-2$;
        \item a directed path from $A$ to $B$ with length $l\leq k-1$ and a confounding structure between $A$ and $B$ with lengths $(r\leq k-2, q\leq k-2)$;
    \end{enumerate}
    
    If $\exists\,$ (i) or (ii), pass. Otherwise, set $A\to B$ \Comment*[r]{Proposition~\ref{prop.sdag and mag}}

    \item Randomly pick a dashed edge $A\starto B$, check whether $A(t)\ind_G B(t+k) | \mathbf{M}(t+1), \mathbf{S}(t)$;

    If the \emph{d}-separation holds, remove $A\starto B$. Otherwise, set $A\to B$ \Comment*[r]{Theorem~\ref{thm.iden sdag}}
\end{enumerate}

until all edges in $G$ are solid ones;
\Return $G$
\end{algorithm}

\begin{remark}
Algorithm~\ref{alg:iden sdag} can be generalized to cases where the subsampling factor $k$ is unknown, by skipping the step (a) below line 3.
\end{remark}

\section{Experiment}
\label{sec.exp}

In this section, we evaluate our method on synthetic data and a real-world application, \emph{i.e.}, discovering causal pathways in Alzheimer's disease.

\noindent\textbf{Compared baselines.} \emph{Methods that account for subsampling}:
\begin{enumerate*}[series=tobecont,itemjoin=\,\,]
    \item SVAR-FCI \cite{malinsky2018causal} that extended the FCI algorithm to time series and recovered a MAG over observed variables; \item NG-EM \cite{gong2015discovering} that achieved identifiability with linear non-Gaussianity and used Expectation Maximization (EM) for estimation.
\end{enumerate*}
\emph{Methods that neglect subsampling}: 
\begin{enumerate*}[series=tobecont,itemjoin=\,\,]
    \item Dynotears \cite{pamfil2020dynotears} that extended the score-based Notears \cite{zheng2018dags} algorithm to the time series; \item PC-GCE \cite{assaad2022discovery} that modified the PC algorithm for temporal data with a new information-theoretic conditional independence (CI) test.
\end{enumerate*}

\noindent\textbf{Metrics.} We use the $\mathrm{F}_1$-score, precision, and recall, where precision and recall (\emph{resp.}) measure the accuracy and completeness of identified causal edges, and $\mathrm{F}_1:=2\cdot \frac{\mathrm{precision}\cdot\mathrm{recall}}{\mathrm{precision}+\mathrm{recall}}$.

\noindent\textbf{Implementation details.} The significance level is set to $0.05$. For the MAG recovery, we use the FCI algorithm implemented in the $\mathrm{causallearn}$ package\footnote{\url{https://github.com/py-why/causal-learn}}. Temporal constraints, \emph{e.g.}, causal precedence, time invariance, and self causation are added as the background knowledge. 

\subsection{Synthetic study}
\label{sec.exp sim}

\indent\textbf{Data generation.} We generate radnom summary graphs with the Erdos-Renyi model \cite{erdHos1960evolution}, where the vertex number is set to $5$, the probability of each edge is set to $0.3$. For each graph, we generate temporal data with the structural equation $X_i(t)=\sum_{j\in \mathbf{PA}_i} f_{ij}(X_j(t-1))+N_i$, where the function $f_{ij}$ is randomly chosen from $\{linear,sin,tanh,sqrt\}$, the exogenous noise $N_i$ is randomly sampled from $\{uniform,gauss.,exp.,gamma\}$. We consider different subsampling factors $k=\{2,3,4,5\}$ and sample sizes $n=\{600,800,1000,1200\}$. For each setting\footnote{We also consider different graph scales $d=\{5,15,25,35,45\}$. Please refer to the appendix for details.}, we replicate over $100$ random seeds.

\noindent\textbf{Comparison with baselines.} Figure~\ref{fig.cmp baselines} shows the performance of our method and baselines under different subsampling factors (upper row) and sample sizes (lower row). As we can see, our method significantly outperforms the others in all settings. Specifically, compared with methods that account for subsampling (SVAR-FCI and NG-EM), our method achieves both higher precision and recall, indicating less detection error and missing edges. This advantage can be attributed to the fact that our method enjoys an identifiability guarantee (\emph{v.s.} SVAR-FCI) and meanwhile requires no parametric assumption on the structural model (\emph{v.s.} NG-EM). Compared with methods that ignore subsampling (Dynotears and PC-GCE), our method achieves higher $\mathrm{F}_1$-score, precision, and is comparable in recall. This result shows that our method can effectively reduce the spurious detection induced by unobserved time steps. Besides, we can observe that Dynotears and PC-GCE slightly outperform SVAR-FCI and NG-EM, which again demonstrates the necessity of establishing nonparametric identification in subsampling problems.

\begin{figure}[tp]
    \centering
    \includegraphics[width=.85\textwidth]{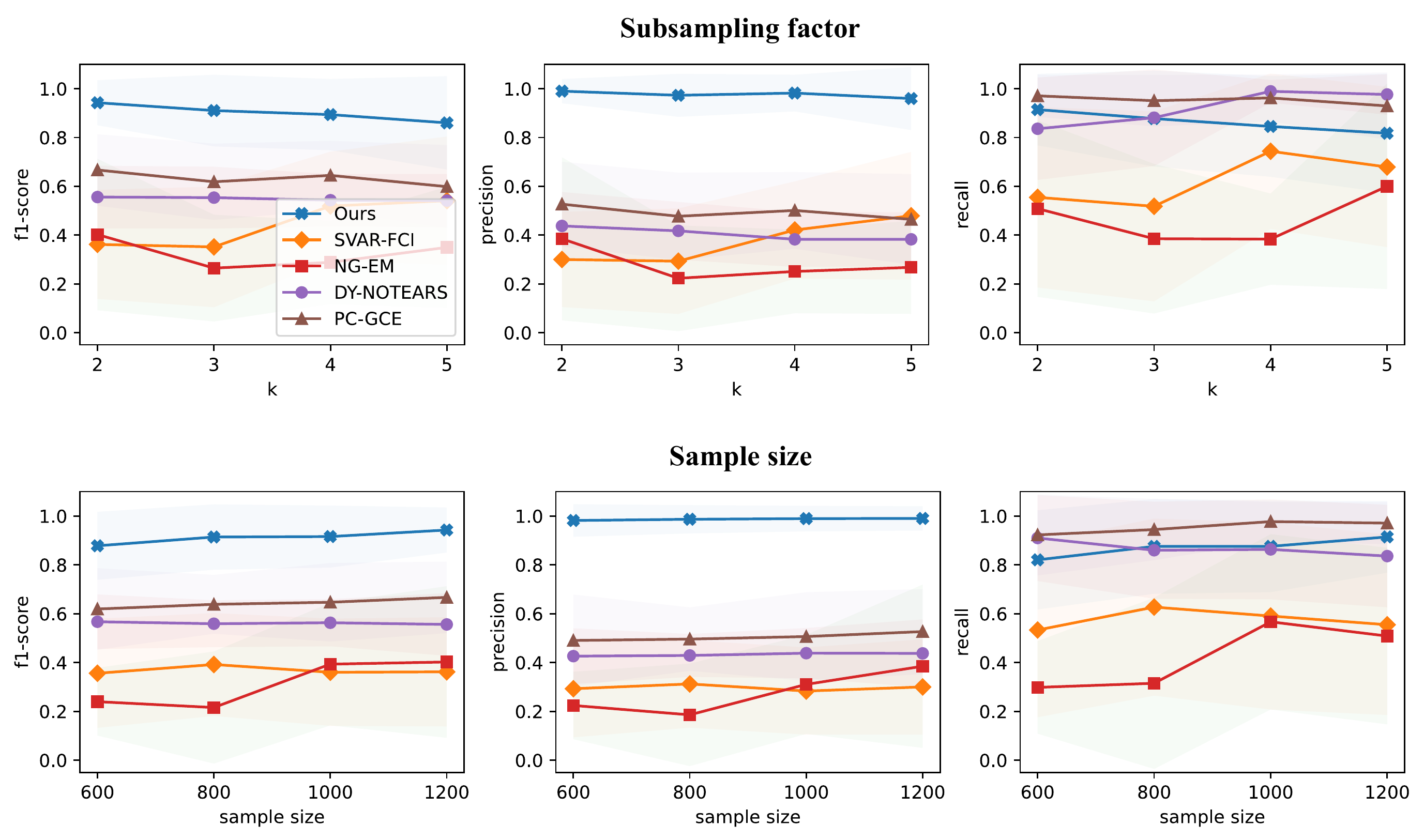}
    \caption{Performance of our method and baselines under different subsampling factors (upper row) and sample sizes (lower row).}
    \label{fig.cmp baselines}
\end{figure}

% \begin{figure}[tp]
%     \centering
%     \includegraphics[width=.65\textwidth]{nips23_fig/intermediate_results.pdf}
%     \caption{Evaluation of the intermediate results. We report the $F_1$-score of the recovered MAG, edges identified by proxies, and the recovered summary graph. For the other metrics, please refer to the appendix.}
%     \label{fig.inter results}
% \end{figure}

\begin{figure}[tp]
    \centering
    \includegraphics[width=.85\textwidth]{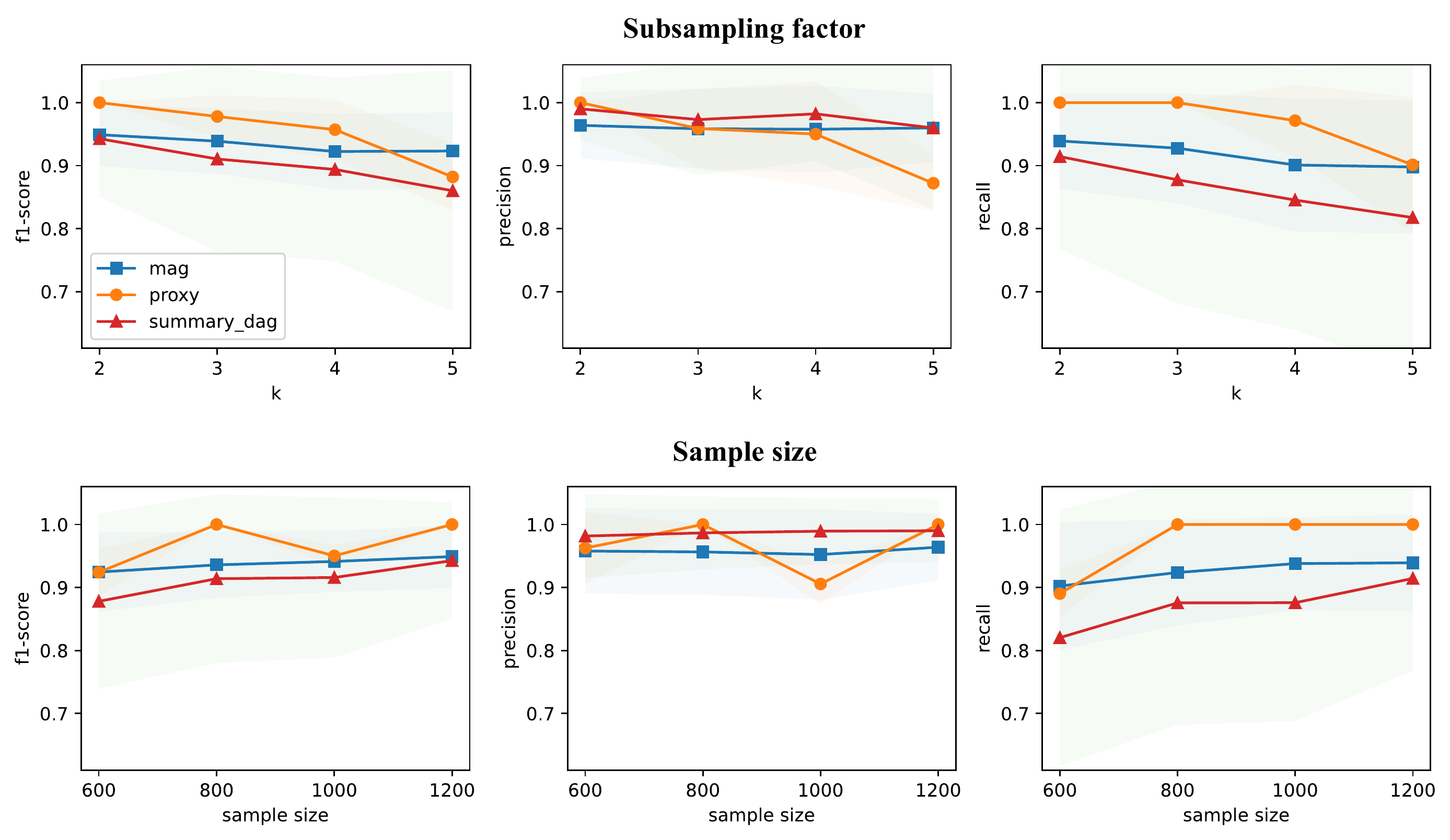}
    \caption{Evaluation of intermediate results. We report $F_1$-score, precision, and recall of the recovered MAG, edges identified by proxies, and the recovered summary graph.}
    \label{fig.inter results}
\end{figure}

\noindent\textbf{Intermediate results.} We evaluate the intermediate results of Algorithm~\ref{alg:iden sdag} (MAG, edges identified by proxies, summary graph) and report the results in Figure~\ref{fig.inter results}. As shown, both the recovered MAG and edges identified by proxies have high accuracy under moderate subsampling factors and sample sizes, hence explaining the effectiveness of our algorithm in recovering the summary graph. Meanwhile, we can also observe that the performance slightly decreases when the subsampling factor $k$ is large. This is explained as follows. When $k$ is large, there are many unobserved steps between two observations, which weakens the correlation pattern in data and therefore breaks our faithfulness assumption.

\subsection{Discovering causal pathways in Alzheimer's disease}
\label{sec.exp ad}

\noindent\textbf{Background.} Alzheimer's disease (AD) is one of the most common neuro-degenerative diseases. In AD, patients suffer from memory loss due to atrophy of memory-related brain regions, such as the Hippocampus \cite{west1994differences} and Temporal lobe \cite{jack1998rate}. A widely accepted explanation for such atrophy is: the disease releases toxic proteins, \emph{e.g.}, $A\beta$ \cite{laferla2007intracellular} and ${tau}$ \cite{kolarova2012structure}; along anatomical pathways, these proteins spread from one brain region to another, eventually leading to atrophy of the whole brain \cite{bloom2014amyloid}. Recovering these anatomical pathways, \emph{i.e.}, underlying causal mechanisms, will benefit the understanding of AD pathology and inspire potential treatment methods.

\noindent\textbf{Dataset and preprocessing.} We consider the Alzheimer's Disease Neuroimaging Initiative (ADNI) dataset \cite{petersen2010alzheimer}, in which the imaging data is acquired from structural Magnetic Resonance Imaging (sMRI) scans. We apply the Dartel VBM \cite{ashburner2010vbm} for preprocessing and the Statistical Parametric Mapping (SPM) \cite{friston2003statistical} for segmenting brain regions. Then, we implement the Automatic Anatomical Labeling (AAL) atlas \cite{tzourio2002automated} to partition the whole brain into $90$ regions. In total, we use $n=558$ subjects with baseline and month-6 follow-up visits enrolled in ADNI-GO/1/2/3 periods.

\begin{figure}[tp]
    \centering
    \includegraphics[width=.85\textwidth]{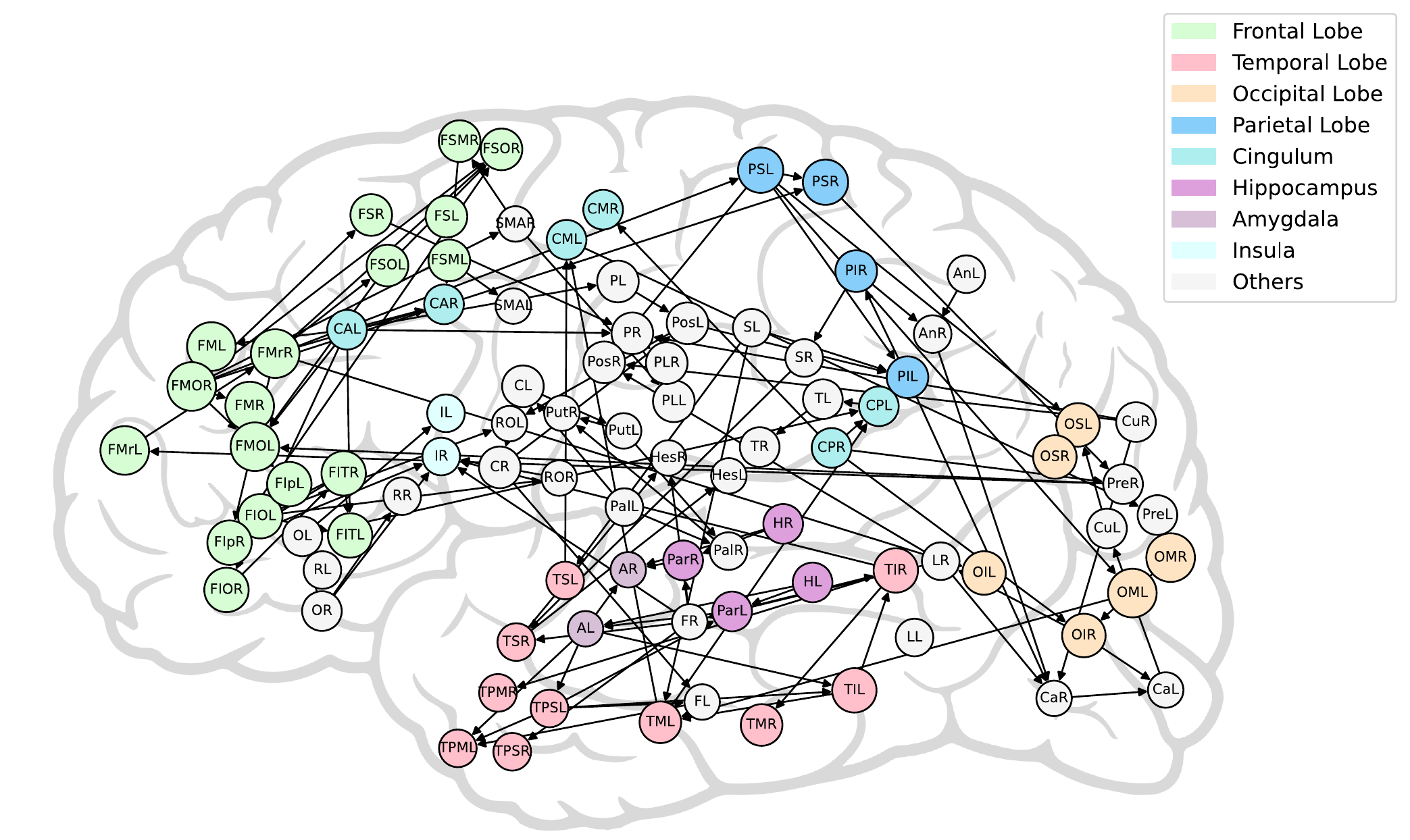}
    \caption{Recovered summary graph in Alzheimer's disease.}
    \label{fig.ad_90nodes}
\end{figure}

\noindent\textbf{Results.} Figure~\ref{fig.ad_90nodes} shows the recovered summary graph over $90$ brain regions. For illustration, in Figure~\ref{fig.ad_8nodes} (a), we further group the brain regions into $8$ meta-regions according to their anatomical structures. We can observe that \textbf{i)} most causal pathways between meta-regions are unidirectional; \textbf{ii)} the identified sources of atrophy are the Hippocampus, Amygdala, and Temporal lobe, which are all early-degenerate regions \cite{west1994differences,jack1998rate}; \textbf{iii)} the topology order induced from our result, \emph{i.e.}, $\{Hip.,Amy.,Tem.,Cin.,Ins.,Fro.,Par.,Occ.\}$ coincides with the temporal degeneration order found in \cite{young2018uncovering,vogel2021four}. These results constitute an important finding that can be supported by existing studies: the atrophy (releasing of toxic proteins) is sourced from the Hippocampus and gradually propagates to other brain structures along certain anatomical pathways \cite{bloom2014amyloid,liu2023whichinv}.

In contrast, in Figure~\ref{fig.ad_8nodes} (b), the causal relations recovered by the Dynotears baseline\footnote{Please refer to the appendix for results of the other baselines.} are less clear. For example, most of the identified interactions are bidirectional, which may be due to the spurious edges induced by subsampling. Besides, the identified atrophy source is the Frontal lobe, with AD-related regions identified as outcomes, which is inconsistent with clinical studies. These results, from another perspective, show the importance of modeling subsampling in time series causal discovery.

\section{Conclusion}
\label{sec.conclusion}

In this paper, we propose a causal discovery algorithm for subsampled time series. Our method leverages the recent progress of proximal causal discovery and can achieve complete identifiability without any parametric assumption. The proposed algorithm can outperform baselines on synthetic data and recover reasonable causal pathways in Alzheimer's disease. 

\noindent\textbf{Limitation and future work.} Our method relies on an accurate test of the conditional independence (CI) and therefore may suffer from low recall when the CI patterns in data are weak. To solve this problem, we will investigate theories on high-efficiency CI tests and pursue a dedicated solution.

\begin{ack}
    This work was supported by National Key R\&D Program of China (2022ZD0114900).
\end{ack}

\bibliographystyle{unsrt}
\bibliography{reference} 

% add the following to make toc work
\newpage
\appendix
\renewcommand{\contentsname}{Appendix}
\tableofcontents
\newpage
\addtocontents{toc}{\protect\setcounter{tocdepth}{2}}

\section{Preliminary}

\newcommand{\figillustrationgraphs}{2}
\newcommand{\remardmag}{2.5}
\newcommand{\thmproxy}{2.8}
\newcommand{\asmsmoothness}{2.6}
\newcommand{\asmmarkandfaith}{2.2}

\subsection{Explanation of Figure~\figillustrationgraphs (c): MAG}

\begin{table}[htp]
    \caption{Detailed explanation of edges in the MAG}
    \centering
    \begin{tabular}{c c}
    \toprule
    Edge in the MAG & Inducing path in the full time DAG\\
    \midrule
    $A(t_1)\to A(t_3)$ & $A(t_1)\to A(t_2)\to A(t_3)$ \\ 
    \midrule
    \multirow{2}{*}{$A(t_1)\to M(t_3)$} & $A(t_1)\to A(t_2)\to M(t_3)$ \\
                        &  $A(t_1)\to M(t_2)\to M(t_3)$ \\
    \midrule
    $A(t_1)\to B(t_3)$ & $A(t_1)\to M(t_2)\to B(t_3)$ \\
    \midrule
    $M(t_1)\to M(t_3)$ & $M(t_1)\to M(t_2)\to M(t_3)$\\
    \midrule
    \multirow{2}{*}{$M(t_1)\to B(t_3)$} & $M(t_1)\to M(t_2)\to B(t_3)$\\
                 & $M(t_1)\to B(t_2)\to B(t_3)$ \\
    \midrule
    $B(t_1)\to B(t_3)$ & $B(t_1)\to B(t_2)\to B(t_3)$ \\
    \midrule
    $A(t_3)\leftrightarrow M(t_3)$ & $A(t_3)\leftarrow A(t_2)\to M(t_3)$ \\
    \midrule
    $M(t_3)\leftrightarrow B(t_3)$ & $M(t_3)\leftarrow M(t_2)\to B(t_3)$ \\
    \bottomrule
    \end{tabular}
\end{table}

\subsection{Proof of Remark~\remardmag: Edge orientations in the MAG}

\noindent\textbf{Remark~\remardmag.}\emph{In the MAG, the directed edge $A \to B$ means $A$ is the ancestor of $B$. The bidirected edge $A \leftrightarrow B$ means there is an unobserved confounder $U$ between $A$ and $B$.}

\begin{proof}
    Since the meaning of the directed edge is directly induced from the definition, we focus on explaining the bidirected edge.

    Specifically, $A \leftrightarrow B$ means there is an inducing path $p$ relative to $\mathbf{L}$ between $A$ and $B$. Since $A \not \in \mathbf{An}(B)$ and $B \not \in \mathbf{An}(A)$, the inducing path $p$ must contains non-mediators, \emph{i.e.}, colliders or confounders. Suppose that there are $r$ confounders on $p$, since between every two confounders there is a collider, the number of colliders on $p$ is $r-1$. Denote the confounders on $p$ as $U_1,...,U_r$, the colliders as $C_1,...,C_{r-1}$, we have $p = A \doubleleftarrows U_1 \doubleto C_1 \doubleleftarrows \cdots \doubleto C_{h-1} \doubleleftarrows U_h \doubleto B$. According to the definition of inducing path, each collider $C_i$ on $p$ is the ancestor of either $A$ or $B$. Based on this, the following algorithm is assured to find a latent confounder between $A$ and $B$.

    \SetKwComment{Comment}{/* }{ */}
    \begin{algorithm}
    \caption{Find the latent confounder}
    \label{alg:A<->B lat confounder}
    \For{$i=1,...,r-1$}{
        \If{$C_i\in \mathbf{An}(B)$}{
            \Return $U_i$ \Comment*[r]{ $A \doubleleftarrows C_{i-1}\doubleleftarrows U_i \doubleto C_i \doubleto B$ }
        }
    }
    \Return $U_r$ \Comment*[r]{ $A \doubleleftarrows C_{r-1}\doubleleftarrows U_r \doubleto B$ }
    \end{algorithm}
\end{proof}

\newpage
\section{Structural identifiability}

\newcommand{\propidenmag}{3.1}
\newcommand{\eqsvarmodel}{1}
\newcommand{\asmselfcause}{2.10}
\newcommand{\propsdagandmag}{3.3}
\newcommand{\thmidensdag}{3.5}

\subsection{Proof of Proposition~\propidenmag: Identifiability of the MAG}

\noindent\textbf{Proposition \propidenmag.} \emph{Assuming model (\eqsvarmodel) and Assumption~\asmmarkandfaith, then the MAG over the observed variable set $\mathbf{O}$ is identifiable, i.e., its skeleton and edge orientations can be uniquely derived from the joint distribution $\mathbb{P}(\mathbf{O})$.}

To prove Proposition~\propidenmag, we first introduce the following lemma, which studies the property of the path between a vertex at $t$ and a vertex at $t+\tau$ ($\tau \geq 0$).

\begin{lemma}
Suppose that there is a path between $A(t),B(t+\tau)$ ($\tau \geq 0$) denoted as $p=A(t),V_1,...,V_r,B(t+\tau)$, if $\exists\, V_i$ on $p$ such that $time(V_i)\geq t+\tau$, then there is a collider on $p$ at $time\geq t+\tau$.
\label{lemma.future path collider}
\end{lemma}

\begin{proof}
We discuss two possible cases and show that $p$ always contains a collider at $time\geq t+\tau$.
\begin{enumerate}[label=\arabic*.,leftmargin=*]
    \item $\forall\, V_i, time(V_i)\leq t+\tau$. In this case, there is a vertex $V_j$ on $p$ such that $time(V_j)=t+\tau$. Consider its adjacent vertices $V_{j-1},V_{j+1}$ on the path. According to our assumption, we have $time(V_{j-1})=time(V_{j+1})=t+\tau-1$. Hence, according to the causal precedence assumption, the edges among $V_{j-1},V_j,V_{j+1}$ are $V_{j-1}(t+\tau-1) \to V_j(t+\tau) \leftarrow V_{j+1}(t+\tau+1)$, which means $V_j(t+\tau)$ is a collider on $p$ at $time=t+\tau$.
    \item $\exists\,V_i, time(V_i)>t+\tau$. Let $j:=\arg\max_{i}time(V_i)$ and $t^*:=time(V_j)$. We have $t^* >t+\tau$. Consider the vertices $V_{j-1},V_j,V_{j+1}$, the edges among them are $V_{j-1}(t^* -1) \to V_j(t^*) \leftarrow V_{j+1}(t^* -1)$. Hence, $V_j(t^*)$ is a collider on $p$ at $time>t+\tau$.
\end{enumerate}
\end{proof}
Equipped with Lemma~\ref{lemma.future path collider}, we now introduce the proof of Proposition~\propidenmag below:

\newenvironment{proof_prop_iden_mag}{
  \renewcommand{\proofname}{Proof of Proposition~\propidenmag}\proof}{\endproof}
\begin{proof_prop_iden_mag}
According to \citep{zhang2008completeness}, the skeleton of the MAG is identifiable under Assumption~\asmmarkandfaith. 

The orientation of edges in the MAG is either directed ($\to$) or bi-directed ($\leftrightarrow$). Since we assume that the cause must precedes the effect, instantaneous edges in the MAG must be bi-directed. In the following, we show that lagged edges in the MAG must be directed and cross $k$ time steps, \emph{e.g.}, $A(t) \to B(t+k)$.

Suppose that $A(t_1)$ and $B(t_2)$ ($t_2>t_1$) are adjacent in the MAG. We first prove that $t_2=t_1+k$. Prove by contradiction. Since $A(t_1)$ and $B(t_2)$ are adjacent in the MAG, there is an inducing path $p=A(t_1),V_1,...,V_r,B(t_2)$ between them in the full time DAG. Suppose that $t_2\neq t_1+k$, in other words, $t_2 \geq t_1+2k$. Since the order of the SVAR process is $1$, there exist $V_i$ on $p$ such that $time(V_i)=t_1+k, time(V_{i+1})=t_1+k+1$. Hence, $V_i(t_1+k)$ is neither latent nor a collider\footnote{$V_i(t_1+k) \to V_{i+1}(t_1+k+1)$ means $V_i(t_1+k)$ can not be a collider.}. As a result, the path $p$ is not an inducing path, which is a contradiction. 

Next, we prove that the edge between $A(t),B(t+k)$ is directed. We prove this by showing that $A(t) \leftrightarrow B(t+k)$ can not be true. Prove by contradiction. Suppose that we have $A(t) \leftrightarrow B(t+k)$ in the MAG, then there is an inducing path $p=A(t),V_1,...,V_r,B(t+k)$ between $A(t),B(t+k)$, and $A(t) \not \in \mathbf{An}(B(t+k))$\footnote{$B(t+k)$ can not be the ancestor of $A(t)$ according to the causal precedence assumption.} in the full time DAG. There are three kinds of possible inducing paths:
\begin{enumerate}[label=\arabic*.]
    \item $\exists\, V_i$ such that $time(V_i)\geq t+k$;
    \item $\forall\, V_i$, $time(V_i)<t+k$ and $time(V_i)>t$;
    \item $\forall\, V_i$, $time(V_i)<t+k$ and $\exists\, V_j$ such that $time(V_j)\leq t$.
\end{enumerate}
In the following, we will show neither of these paths exist, which is a contradiction and shows that the edge between $A(t),B(t+k)$ is directed.
\begin{enumerate}[label=\arabic*.,leftmargin=*]
    \item $\exists\, V_i$ such that $time(V_i)\geq t+k$. In this case, according to Lemma~\ref{lemma.future path collider}, there is a collider on $p$ at $time\geq t+k$. Since the assumed causal precedence, the collider can not be the ancestor of $A(t)$ or $B(t+k+1)$. Hence, the path $p$ is not an inducing path.
    \item $\forall\, V_i$, $time(V_i)<t+k$ and $time(V_i)>t$. In other words, we have $time(V_i)=t+1,..,t+k-1$. We show that in this case, we have $A(t) \in \mathbf{An}(B(t+k))$ in the full time DAG, which contradicts with $A(t) \leftrightarrow B(t+k)$ in the MAG. 
    
    Specifically, when $k=2$, the path is $p=A(t) \to V_1(t+1) \to B(t+2)$. Hence, $A(t) \in \mathbf{An}(B(t+k))$. When $k\geq 3$, consider the first \contour{black}{$\leftarrow$} edge along the path\footnote{If all edges among the path are $\to$, then $A(t) \in \mathbf{An}(B(t+k))$.}, we have $A(t)\doubleto V_{i-1}(t_i+1)$ \contour{black}{$\leftarrow$} $V_i(t_i)$ and $i\geq 3$. Therefore, $V_{i-1}(t_i+1)$ is the descendent of $A(t)$ and a collider on $p$. According to the definition of the inducing path, $V_{i-1}(t_i+1)$ is the ancestor of $B(t+k)$, which means $A(t) \in \mathbf{An}(B(t+k))$.
    \item $\forall\, V_i$, $time(V_i)<t+k$ and $\exists\, V_j$ such that $time(V_j)\leq t$. In this case, we further consider two possibilities,
    \begin{enumerate}[label=\alph*.]
        \item $\forall\,V_i, t\leq time(V_i) < t+k$. In this case, there is a vertex $V_j$ on $p$ such that $time(V_j)=t$ and we have $V_{j-1}(t+1) \leftarrow V_j(t) \to V_{j+1}(t+1)$. Therefore, $V_j(t)$ is neither a latent variable nor a collider, which means $p$ is not an inducing path.
        \item $\forall\, V_i, time(V_i) < t+k$ and $\exists\, V_i$ such that $time(V_i)<t$. In this case, we have $p=A(t),...,V_i(t-\tau),...,B(t+k)$ ($\tau>0$). As a result,there is a vertex $V_j$ on $p$ between $V_i(t-\tau)$ and $B(t+k)$ such that $time(V_j)=t$ and $time(V_{j+1}=t+1$. Therefore, $V_j(t)$ is neither a latent variable nor a collider, which means $p$ is not an inducing path.
    \end{enumerate}
\end{enumerate}
To conclude, we have shown that if there is an inducing path between $A(t),B(t+k)$, then $A(t) \in \mathbf{An}(B(t+k))$, which means all lagged edges in the MAG are directed ones.
\end{proof_prop_iden_mag}

\subsection{Proof of Proposition~\propsdagandmag: MAG to summary DAG}

\noindent\textbf{Proposition~\propsdagandmag.}\emph{If there are $A(t) \to B(t+k)$ and $A(t+k) \leftrightarrow B(t+k)$ in the MAG, then, in the summary DAG, there is either 
\begin{enumerate}[leftmargin=25pt]
    \item $A \to B$,
    \item a directed path from $A$ to $B$ with length $l\leq k-2$, or
    \item a directed path from $A$ to $B$ with length $l=k-1$ and a confounding structure between them with lengths $(r\leq k-2,q\leq k-2)$.
\end{enumerate}}

\begin{proof}
\begin{enumerate}[label=\arabic*.,leftmargin=*]
    \item We first show that If there is $A\to B$ in the summary DAG, then, there are $A(t) \to B(t+k)$ and $A(t+k) \leftrightarrow B(t+k)$ in the MAG: 
    
    Suppose that there is $A\to B$ in the summary DAG. Then, there is a directed inducing path $p_1=A(t)\to A(t+1)\doubleto A(t+k-1)\to B(t+k)$ from $A(t)$ to $B(t+k)$ in the full time DAG. Thus, we have $A(t) \to B(t+k)$ in the MAG. 
    
    Besides, there is an inducing path $p_2=A(t+k) \leftarrow A(t+k-1) \to B(t+k)$ in the full time DAG, which means we have $A(t+k) \leftrightarrow B(t+k)$ in the MAG.
    
    \item We then show that if there is $A(t) \to B(t+k)$ in the MAG, then, there is a directed path from $A$ to $B$ with length $0\leq l\leq k-1$ (a directed path with $l=0$ means $A\to B$) in the summary DAG:
    
    Suppose that there is $A(t) \to B(t+k)$ in the MAG. Then, there is a directed path $A(t)\to V_1(t+1) \doubleto V_{k-1}(t+k-1) \to B(t+k)$ in the full time DAG. Denote the number of $V_i$ such that $V_i=A$ as $r$, then, we have $0\leq r \leq k-1$. Hence, in the summary DAG, there is a directed path from $A$ to $B$ with length $l=k-r-1$ and we have $0 \leq l \leq k-1$.
    
    \item We finally show that if there is $A(t+k)\leftrightarrow B(t+k)$ in the MAG, then, in the summary DAG, at least one of the following structures exists:
    \begin{enumerate}[label=\alph*.]
        \item A directed path\footnote{The path can be from $A$ to $B$ or from $B$ to $A$.} $p$ between $A$ and $B$ of length $0\leq l\leq k-2$ (a directed path with $l=0$ means $A\to B$);
        \item A confounding structure $c$ between $A$ and $B$ of length $(r\leq k-2, q\leq k-2)$.
    \end{enumerate}
    
    Suppose that there is $A(t+k)\leftrightarrow B(t+k)$ in the MAG, we will show that there is a latent confounder $U$ between $A(t+k)$ and $B(t+k)$ such that $t+1\leq time(U)\leq t+k-1$. 
    
    In this regard, if $U=A$ or $U=B$, there is a directed path between $A$ and $B$ with length $0\leq l\leq k-2$; Otherwise when $U\neq A$ and $U\neq B$, there is a confounding structure between $A$ and $B$ with length $(r\leq k-2,q\leq k-2)$.

    In the following, we prove that there is a latent confounder $U$ between $A(t+k)$ and $B(t+k)$ such that $t+1\leq time(U)\leq t+k-1$. Prove by contradiction. Suppose that $time(U)\leq t$ or $time(U)\geq t+k$. If $time(U)\leq t$, then on the directed inducing path from $U$ to $A$, there is a vertex at time $t$, which means the path can not be an inducing path. If $time(U)\geq t+k$, then there is a collider on the inducing path between $U$ and $A$ according to Lemma~\ref{lemma.future path collider}, which means the path is not directed and $U$ is not a latent confounder. Hence, the time of the latent confounder $U$ is between $t+1$ and $t+k-1$.
\end{enumerate}

To conclude, combining the results in 2. and 3., we prove the proposition.
\end{proof}

\subsection{Proof of Theorem~\thmidensdag: Identifiability of the summary DAG}

For two vertex $A,B$, let the vertex set $\mathbf{M}$ contain $A$ and any $M_i\neq B$ such that $A(t)\to M_i(t+k)$ in the MAG and $M_i$ is not $B$'s descendant\footnote{This can be justified from the MAG. Specifically, if $M_i$ is the descendant of $B$, then there is $B(t) \to V_1(t+k), V_1(t) \to V_2(t+k), ..., V_l(t)\to M_i(t+k)$ in the MAG.}. Let the vertex set $\mathbf{S}$ contains any $S_i\neq A$ such that $S_i(t) \to B(t+k)$ or $S_i(t) \to M_j(t+k)$ for some $M_j \in \mathbf{M}$ in the MAG.

\noindent\textbf{Theorem~\thmidensdag.} \emph{Assuming model (\eqsvarmodel), Assumption~\asmmarkandfaith, and Assumptions~2.6,2.7,2.11, then the summary DAG is identifiable. Specifically,
\begin{enumerate}[leftmargin=25pt,itemsep=0.15mm]
    \item There is $A\to B$ in the summary DAG iff there are $A(t)\to B(t+k), A(t+k)\leftrightarrow B(t+k)$ in the MAG, and the set $\mathbf{M}(t+1)\cup \mathbf{S}(t)$ is not sufficient to d-separate $A(t),B(t+k)$ in the full time DAG. 
    \item The condition ``the set $\mathbf{M}(t+1)\cup \mathbf{S}(t)$ is not sufficient to d-separate $A(t),B(t+k)$ in the full time DAG'' can be tested by the proxy variable $\mathbf{M}(t+k)$ of the unobserved set $\mathbf{M}(t+1)$.
\end{enumerate}}

\begin{proof}
\begin{enumerate}[label=\arabic*.,leftmargin=*]
    \item $\Rightarrow$. Suppose that there is $A \to B$ in the summary DAG. According to Proposition~\propsdagandmag, we have $A(t) \to B(t+k)$ and $A(t+k) \leftrightarrow B(t+k)$ in the MAG. In addition, in the full time DAG, there is a directed path $p=A(t)\to B(t+1)\doubleto B(t+k)$ from $A(t)$ to $B(t+k)$, which is not \emph{d}-separated by $\mathbf{M}(t+1)\cup \mathbf{S}(t)$ (because $B \not \in \mathbf{M}$).
    
    $\Leftarrow$. Suppose that there are $A(t) \to B(t+k)$ and $A(t+k) \leftrightarrow B(t+k)$ in the MAG. According to Proposition~\propsdagandmag, at least one of the following structures exists:
    \begin{enumerate}[label=\alph*.]
        \item $A\to B$;
        \item A directed path $p_{AB}$ from $A$ to $B$ of length $0<l\leq k-2$;
        \item A directed path $p_{AB}$ from $A$ to $B$ of length $l=k-1$ and a confounding structure $c_{AB}$ between $A$ and $B$ of length $(r\leq k-2,q\leq k-2)$.
    \end{enumerate}
    In the following, we show that if the set $\mathbf{M}(t+1)\cup \mathbf{S}(t)$ is not sufficient to \emph{d}-separate $A(t),B(t+k)$ in the full time DAG, $A\to B$ must exist in the summary DAG. We prove this by its contrapositive statement, \emph{i.e.}, if $A\to B$ does not exist, the set $\mathbf{M}(t+1)\cup \mathbf{S}(t)$ is sufficient to \emph{d}-separate $A(t)$ and $B(t+k)$.

    Consider the path between $A(t)$ and $B(t+k)$ in the full time DAG $p=A(t),V_1,...,V_l,B(t+k)$. There are three possible cases:
    \begin{enumerate}[label=\alph*.]
        \item $\exists\, V_i, time(V_i)\geq t+k$;
        \item $\forall\, V_i, time(V_i)<t+k$ and $\exists\, V_j$ such that $time(V_j)\leq t$;
        \item $\forall\, V_i, t<time(V_i)<t+k$.
    \end{enumerate}
    Next, we will show that if $A\to B$ does not exist, the set $\mathbf{M}(t+1)\cup \mathbf{S}(t)$ can \emph{d}-separate all these three kinds of paths.
    \begin{enumerate}[label=\alph*.]
        \item For the path that $\exists\, V_i, time(V_i)\geq t+k$. According to Lemma~\ref{lemma.future path collider}, there is a collider at $time\geq t+k$. Because the collider and its descendants are not in $\mathbf{M}(t+1)\cup \mathbf{S}(t)$, the path is \emph{d}-separated.
        \item For the path that $\forall\, V_i, time(V_i)<t+k$ and $\exists\, V_j$ such that $time(V_j)\leq t$. There is a vertex $V_j$ on the path such that $V_j(t)\to V_{j+1}(t+1) \to V_{j+2}(t+2)\to\!\cdots$\footnote{Strictly, when $k=2$, $V_{j+2}$ is $B(t+2)$.} If the path from $V_j(t)$ to $B(t+k)$ is a directed path, we have $V_j\in \mathbf{S}$ and the path is \emph{d}-separated. Otherwise, there is a collider on the path at $time\geq t+2$. Neither the collider nor its descendants are in $\mathbf{M}(t+1)\cup \mathbf{S}(t)$, the path is also \emph{d}-separated.
        \item For the path that $\forall\, V_i, t<time(V_i)<t+k$. We have $A(t)\to V_1(t+1) \to V_2(t+2) \to\!\cdots$. If the path from $V_1(t+1)$ to $B(t+k)$ is a directed one, $V_1\in \mathbf{M}$\footnote{$A(t)\to V_1(t+1)$ means we have $A(t) \to V_1(t+k)$ in the MAG. $V_1$ is the ancestor of $B$, so $V_1 \not\in \mathbf{De}(B)$. Hence, $V_i\in \mathbf{M}$.}. Otherwise, there is a collider on the path at $time\geq t+2$. Neither the collider nor its descendants are in $\mathbf{M}(t+1)\cup \mathbf{S}(t)$, the path is also \emph{d}-separated.
    \end{enumerate}

    \item In the following, we show that the condition ``the set $\mathbf{M}(t+1)\cup \mathbf{S}(t)$ is not sufficient to \emph{d}-separate $A(t),B(t+k)$ in the full time DAG'' can be justified by testing $A(t) \ind B(t+k) | \mathbf{M}(t+1)\cup \mathbf{S}(t)$ with proxy variables.
    \begin{enumerate}[label=\alph*.]
        \item We first explain that $A(t),B(t+k),\mathbf{M}(t+1),\mathbf{M}(t+k)$ have the causal graph shown in Fig.~\ref{fig.ABZ proxy causal graph}.

        \begin{figure}
            \centering
            \includegraphics[width=.4\textwidth]{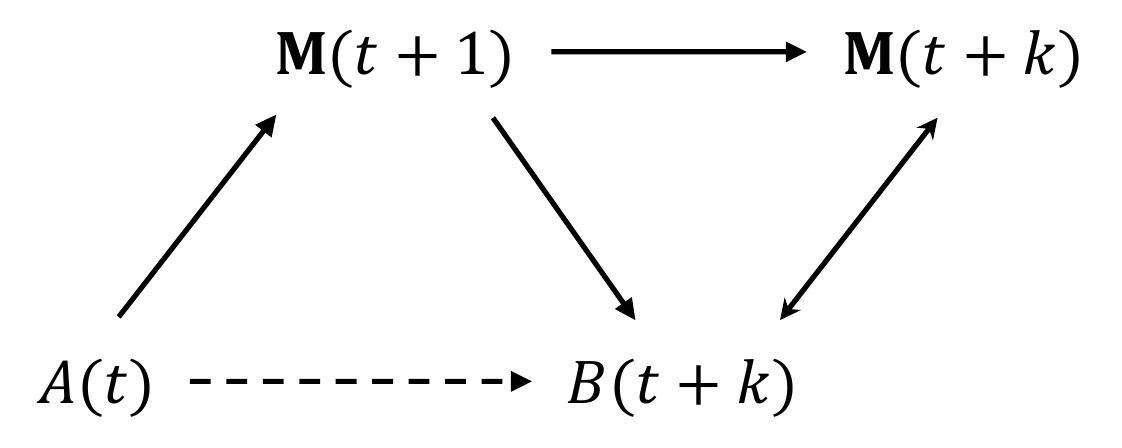}
            \caption{Causal graph over $A(t),B(t+k),\mathbf{M}(t+1),\mathbf{M}(t+k)$.}
            \label{fig.ABZ proxy causal graph}
        \end{figure}

        \begin{itemize}
            \item $A(t)\to \mathbf{M}(t+1)$. By definition, $A$ is the ancestor of vertices in $\mathbf{M}$.
            \item $\mathbf{M}(t+1) \to \mathbf{M}(t+k)$. By Assumption~\asmselfcause the self causation edge always exists.
            \item $\mathbf{M}(t+1) \to B(t+k)$. By definition, $B$ is not the ancestor of any vertex in $\mathbf{M}$.
            \item $\mathbf{M}(t+1) \leftrightarrow B(t+1)$. According to Proposition~\propidenmag, the instantaneous edge must be bi-directed.
        \end{itemize}
        Indeed, the directed edge $\mathbf{M}(t+1) \to B(t+k)$ and the bi-directed edge $\mathbf{M}(t+k)\leftrightarrow B(t+k)$ may not exist. However, according to the theory of single proxy causal discovery, as long as we have $A(t) \ind \mathbf{M}(t+k) | \mathbf{M}(t+1),\mathbf{S}(t)$ (which will be proved in the following), we can test whether $A(t) \to B(t+k)$ with the proxy variable $\mathbf{M}(t+k)$.
        
        \item We next show that $A(t) \ind \mathbf{M}(t+k) | \mathbf{M}(t+1),\mathbf{S}(t)$, \emph{i.e.}, $\mathbf{M}(t+k)$ can act as a legal proxy variable of $\mathbf{M}(t+1)$. We prove this by showing that all path between $A(t)$ to any $Z_i(t+k)\in \mathbf{M}(t+k)$ can be \emph{d}-separated by $\mathbf{M}(t+1)\cup \mathbf{S}(t+1)$. Specifically, we consider two different cases:
        \begin{itemize}
            \item When $A\not \to B$ in the summary DAG. To prove this conclusion, again, consider three kinds of paths $A(t),V_1,...,V_l,Z_i(t+k)$ between $A(t)$ and $Z_i(t+k)$, 
            \begin{enumerate}[label=\roman*.]
                \item For the path that $\exists\, V_i, time(V_i)\geq t+k$. It must contain a collider at $time\geq t+k$. Hence, neither the collider nor its descendant is in $\mathbf{M}(t+1)\cup\mathbf{S}(t)$ and the path is \emph{d}-separated.
                \item For the path that $\forall\, V_i, time(V_i)<t+k$ and $\exists\, V_j$ such that $time(V_j)\leq t$. Then, there is a vertex $V_j$ on the path such that $V_j(t)\to V_{j+1}(t+1) \to V_{j+2}(t+2)\to\!\cdots$. If the path between $V_j(t)$ and $Z_i(t+k)$ is a directed one, then we have $V_j(t)\in \mathbf{S}(t)$ and the path is \emph{d}-separated. Otherwise, there is a collider on the path between $V_{j+2}(t+2)$ and $Z_i(t+k)$. Thus, neither the collider nor its descendant is in $\mathbf{M}(t+1)\cup\mathbf{S}(t)$ and the path is also \emph{d}-separated.
                \item For the path that $\forall\, V_i, t<time(V_i)<t+k$. In this case, the path must be $A(t)\to V_1(t+1) \to V_2(t+2) \to\!\cdots$. If the path between $V_1(t1)$ and $Z_i(t+k)$ is a directed one, we have $V_1(t+1)\in \mathbf{M}(t+1)$, because of the following facts: $V_1\neq B$, $A\to V_1$, $V_1\not\in\mathbf{De}(B)$\footnote{If $V_1\in \mathbf{De}(B)$, since $V_1$ is the ancestor of $Z_i$, we have $Z_i\in \mathbf{De}(B)$, which contradicts with the definition of the vertex set $\mathbf{M}$.}. As a result, the path is \emph{d}-separated. Otherwise, when the path between $V_1(t1)$ and $Z_i(t+k)$ is not a directed one, there is a collider on the path at $time\geq t+2$. Therefore, neither the collider nor its descendant is in $\mathbf{M}(t+1)\cup\mathbf{S}(t)$ and the path is \emph{d}-separated.
            \end{enumerate}
            \item When $A \to B$ in the summary DAG. In this case, except for the above analysis, we need to extra show that every path $p=A(t)\to B(t+1),V_1,...,V_l,Z_i(t+k)$ can be \emph{d}-separated by $\mathbf{M}(t+1)\cup \mathbf{S}(t)$. Specifically, since we define $Z_i\not\in\mathbf{De}(B)$, $p$ can not be a directed path. Therefore, starting from $B(t+1)$ along the path, at least one of the edges is $\leftarrow$. As a result, there is a collider on $p$ which is either $B(t+1)$ itself or at $time\geq t+2$. For both cases, neither the collider nor its descendant is in $\mathbf{M}(t+1)\cup\mathbf{S}(t)$ and the path is \emph{d}-separated.
        \end{itemize}
    \end{enumerate}
\end{enumerate}    
\end{proof}

\newpage
\section{Experiment}

\newcommand{\figinterresults}{6}
\newcommand{\secexpad}{5.2}
\newcommand{\figadnintynodes}{7}

% \subsection{Full results of Fig.~\figinterresults: Evaluation of intermediate results}

% \begin{figure}[htp]
%     \centering
%     \includegraphics[width=\textwidth]{nips23_fig/intermediate_results_all.pdf}
%     \caption{Evaluation of intermediate results. We report $F_1$-score, precision, and recall of the recovered MAG, edges identified by proxies, and the recovered summary DAG.}
%     \label{appfig.intermediate_results_all}
% \end{figure}

\subsection{Extra results of Section~5.1: Synthetic study}

\begin{figure}[htp]
    \centering
    \includegraphics[width=.85\textwidth]{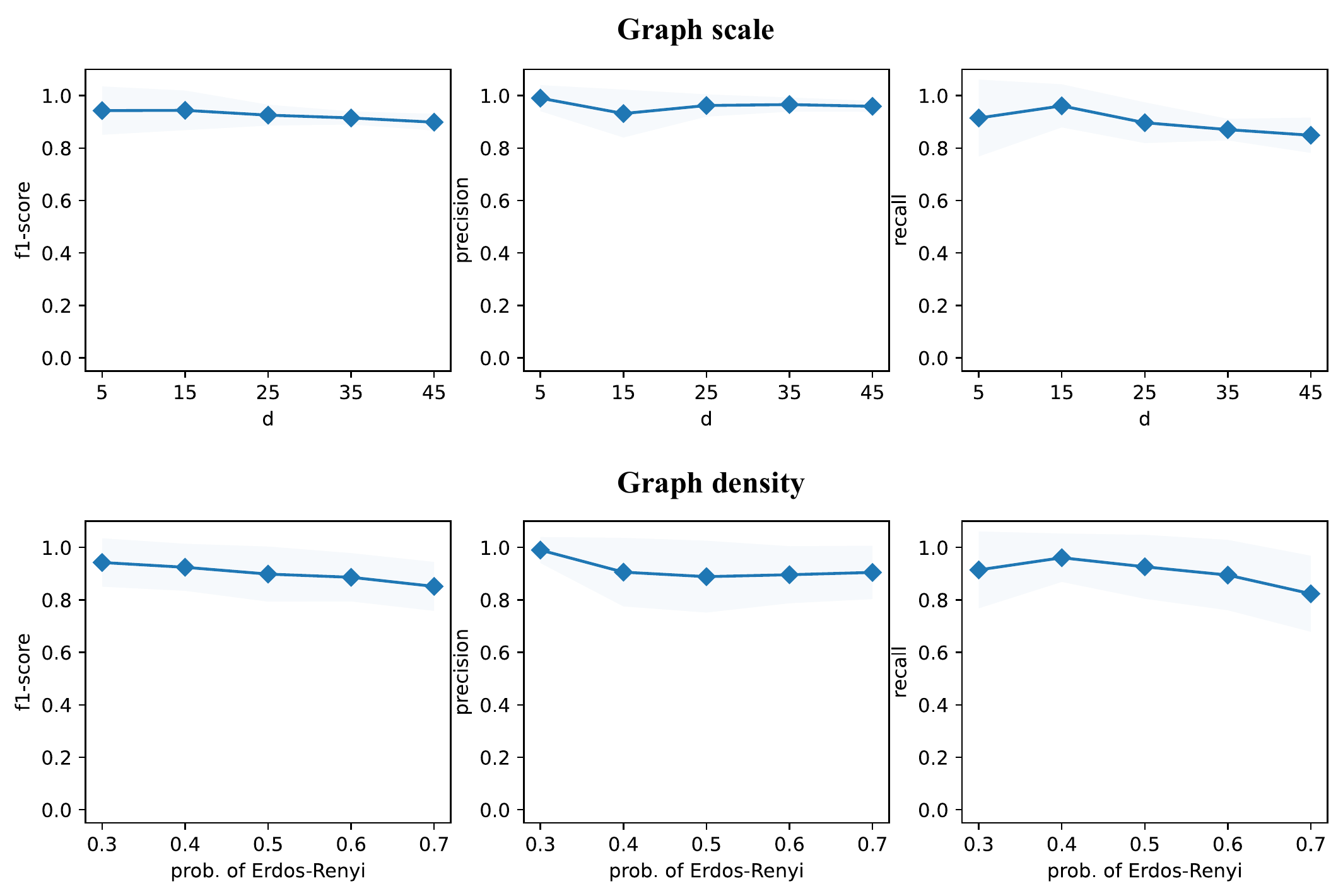}
    \caption{Performance of our method under different graph scales ($d$ denotes the variable number) and densities ($\mathrm{prob.}$ denotes the edge probability in the Erdos-Renyi model).}
\end{figure}

\subsection{Extra results of Section~\secexpad: Discovering causal pathways in Alzheimer's disease}

\begin{figure}[htp]
    \centering
    \includegraphics[width=.65\textwidth]{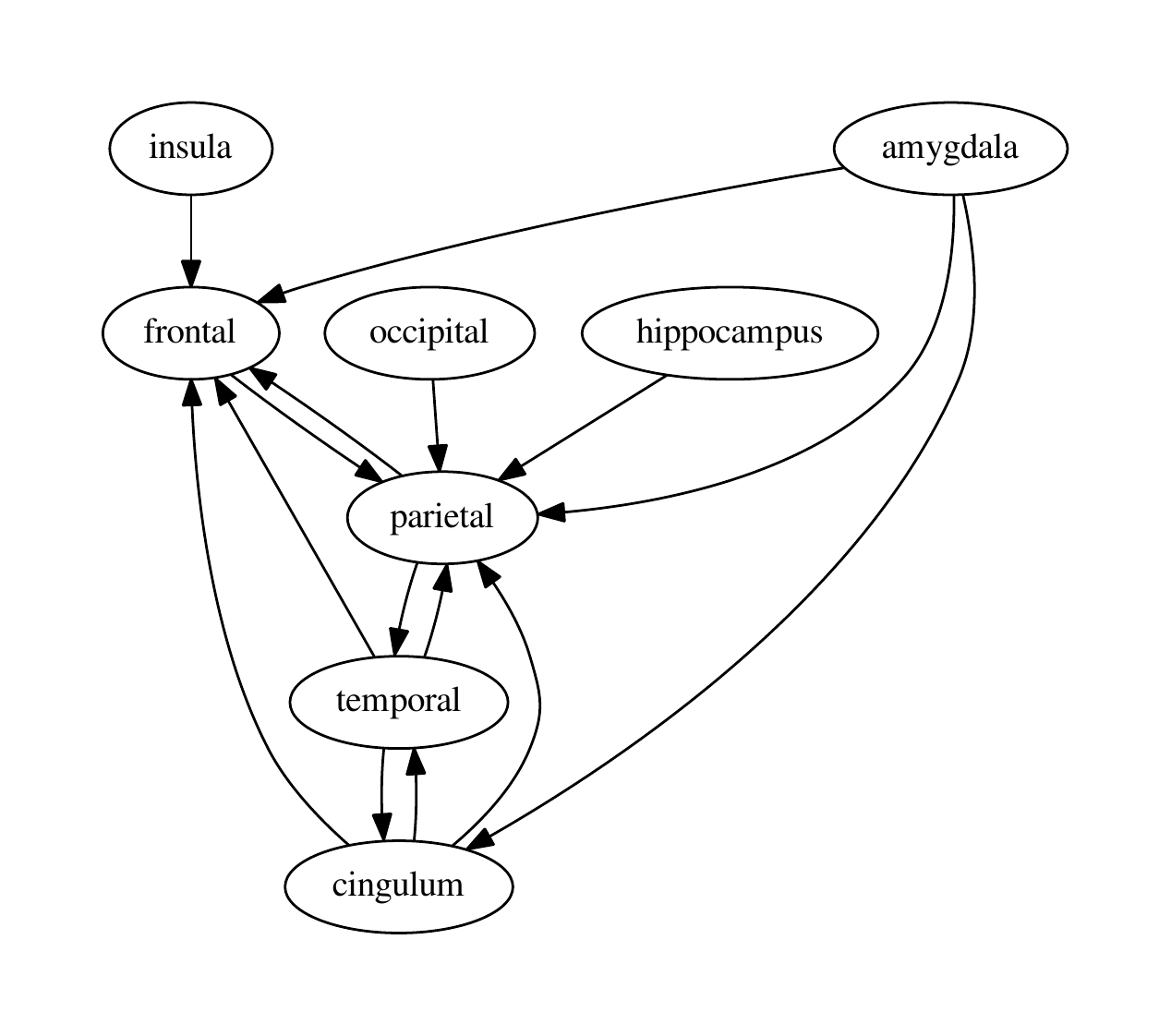}
    \caption{Summary DAG over eight meta-regions in AD recovered by the NG-EM \cite{gong2015discovering} baseline.}
    \label{appfig.ngem8nodes}
\end{figure}

\begin{figure}[htp]
    \centering
    \includegraphics[width=.9\textwidth]{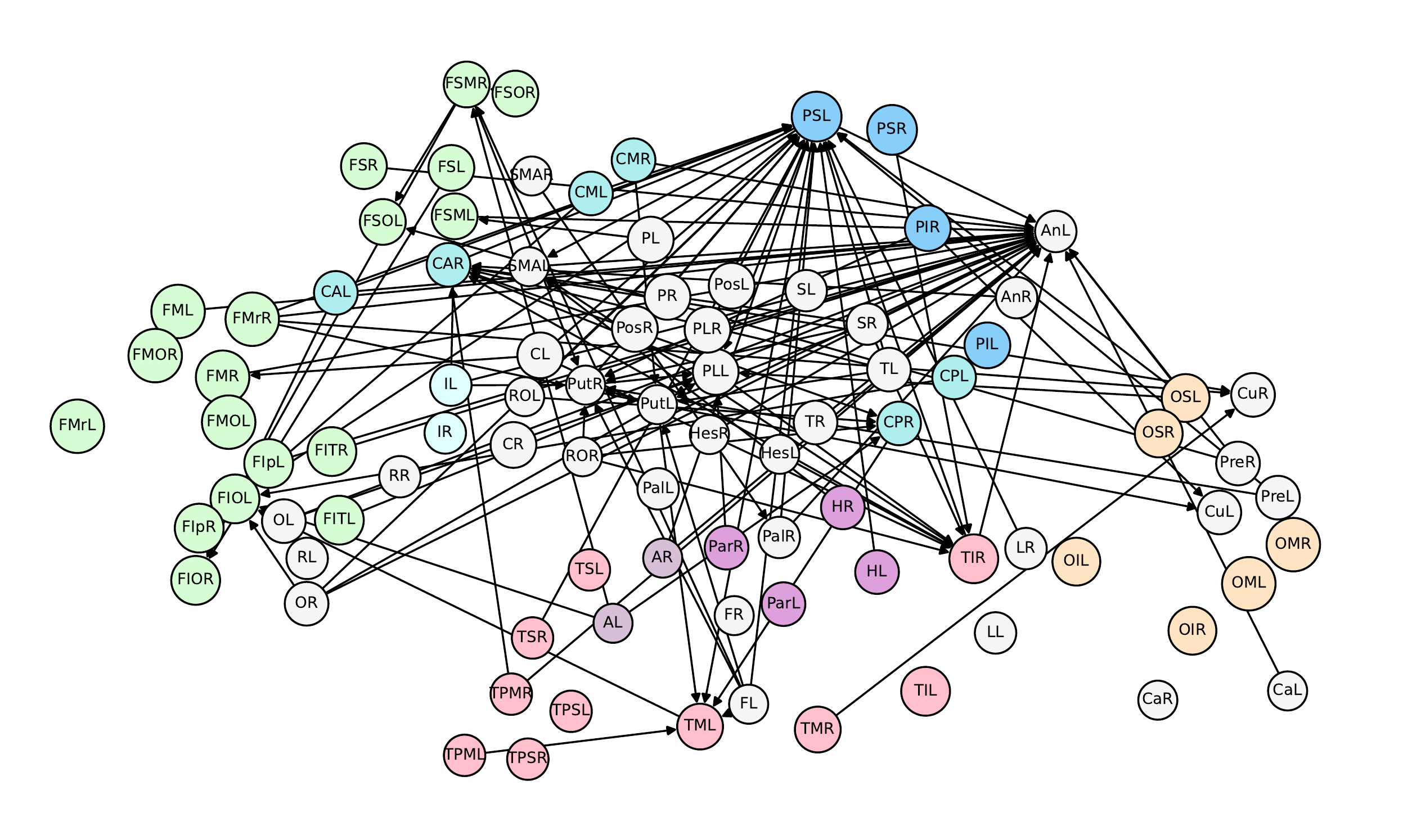}
    \caption{Summary DAG in AD recovered by the NG-EM \cite{gong2015discovering} baseline.}
    \label{appfig.ngem90nodes}
\end{figure}

\begin{figure}[htp]
    \centering
    \includegraphics[width=.9\textwidth]{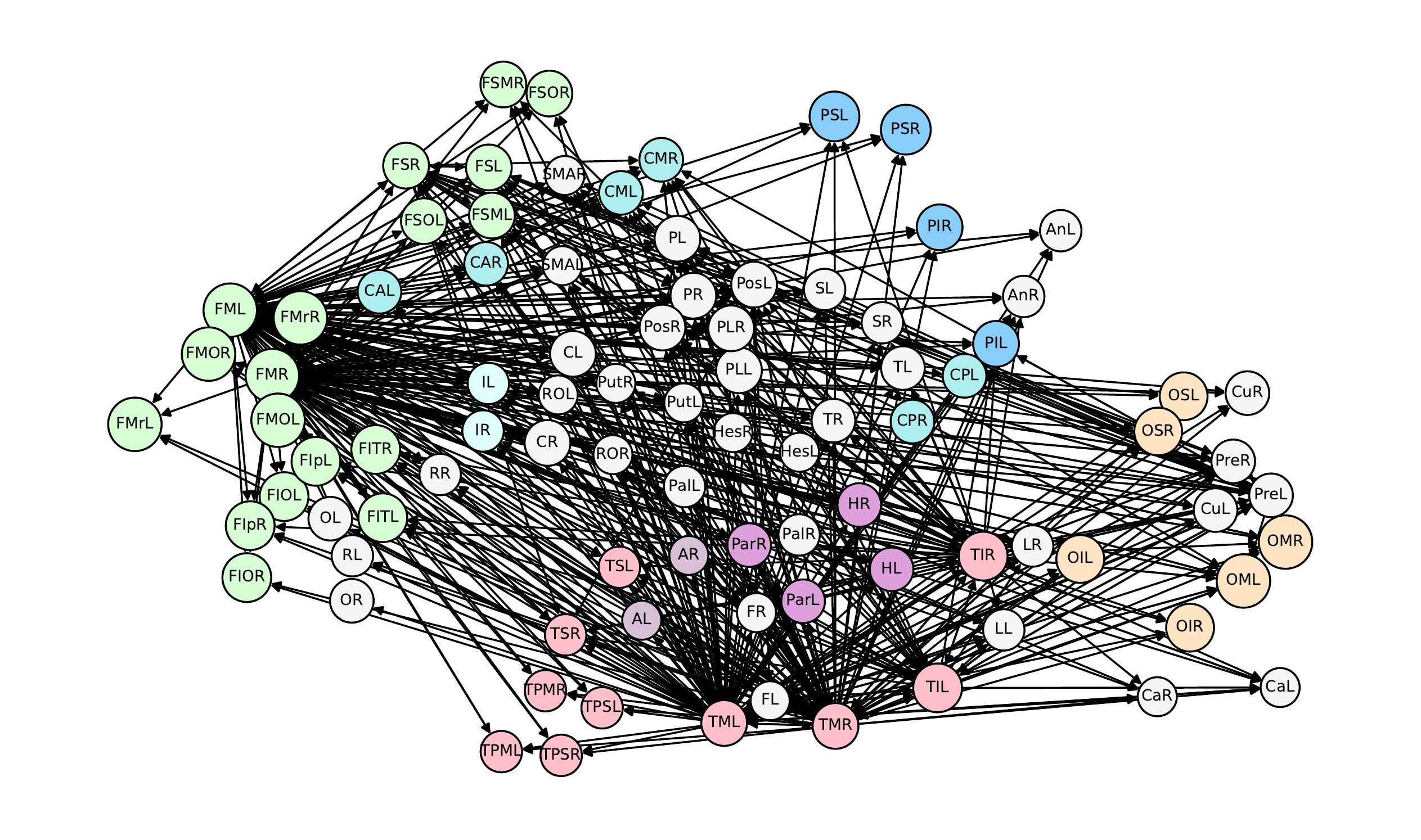}
    \caption{Summary DAG in AD recovered by the Dynotears \cite{pamfil2020dynotears} baseline.}
    \label{appfig.dynotear90nodes}
\end{figure}

\begin{table}[htp]
    \caption{Indices for brain regions partition in Fig.~\figadnintynodes.}
    \label{apptab:brain region abbr}
    \centering
    \resizebox{\textwidth}{!}{ 
    \begin{tabular}{c c c | c c c}
    \toprule
    AAL Index & Abbreviation & Full name & AAL Index & Abbreviation & Full name\\
    \midrule
    1 & PL & Precentral L &  2 & PR & Precentral R\\
     3 & FSL & Frontal Sup L &  4 & FSR & Frontal Sup R\\
     5 & FSOL & Frontal Sup Orb L &  6 & FSOR & Frontal Sup Orb R\\
     7 & FML & Frontal Mid L &  8 & FMR & Frontal Mid R\\
     9 & FMOL & Frontal Mid Orb L1 &  10 & FMOR & Frontal Mid Orb R1\\
     11 & FIOL & Frontal Inf Oper L &  12 & FIOR & Frontal Inf Oper R\\
     13 & FITL & Frontal Inf Tri L &  14 & FITR & Frontal Inf Tri R\\
     15 & FIpL & Frontal Inf Orb L &  16 & FIpR & Frontal Inf Orb R\\
     17 & ROL & Rolandic Oper L &  18 & ROR & Rolandic Oper R\\
     19 & SMAL & Supp Motor Area L &  20 & SMAR & Supp Motor Area R\\
     21 & OL & Olfactory L &  22 & OR & Olfactory R\\
     \midrule
     23 & FSML & Frontal Sup Medial L &  24 & FSMR & Frontal Sup Medial R\\
     25 & FMrL & Frontal Mid Orb L2 &  26 & FMrR & Frontal Mid Orb R2\\
     27 & RL & Rectus L &  28 & RR & Rectus R\\
     29 & IL & Insula L &  30 & IR & Insula R\\
     31 & CAL & Cingulum Ant L &  32 & CAR & Cingulum Ant R\\
     33 & CML & Cingulum Mid L &  34 & CMR & Cingulum Mid R\\
     35 & CPL & Cingulum Post L &  36 & CPR & Cingulum Post R\\
     37 & HL & Hippocampus L &  38 & HR & Hippocampus R\\
     39 & ParL & ParaHippocampal L &  40 & ParR & ParaHippocampal R\\
     41 & AL & Amygdala L &  42 & AR & Amygdala R\\
     43 & CaL & Calcarine L &  44 & CaR & Calcarine R\\
     45 & CuL & Cuneus L &  46 & CuR & Cuneus R\\
     \midrule
     47 & LL & Lingual L &  48 & LR & Lingual R\\
     49 & OSL & Occipital Sup L &  50 & OSR & Occipital Sup R\\
     51 & OML & Occipital Mid L &  52 & OMR & Occipital Mid R\\
     53 & OIL & Occipital Inf L &  54 & OIR & Occipital Inf R\\
     55 & FL & Fusiform L &  56 & FR & Fusiform R\\
     57 & PosL & Postcentral L &  58 & PosR & Postcentral R\\
     59 & PSL & Parietal Sup L &  60 & PSR & Parietal Sup R\\
     61 & PIL & Parietal Inf L &  62 & PIR & Parietal Inf R\\
     63 & SL & SupraMarginal L &  64 & SR & SupraMarginal R\\
     65 & AnL & Angular L &  66 & AnR & Angular R\\
     67 & PreL & Precuneus L &  68 & PreR & Precuneus R\\
     69 & PLL & Paracentral Lobule L &  70 & PLR & Paracentral Lobule R\\
     \midrule
     71 & CL & Caudate L &  72 & CR & Caudate R\\
     73 & PutL & Putamen L &  74 & PutR & Putamen R\\
     75 & PalL & Pallidum L &  76 & PalR & Pallidum R\\
     77 & TL & Thalamus L &  78 & TR & Thalamus R\\
     79 & HesL & Heschl L &  80 & HesR & Heschl R\\
     81 & TSL & Temporal Sup L &  82 & TSR & Temporal Sup R\\
     83 & TPSL & Temporal Pole Sup L &  84 & TPSR & Temporal Pole Sup R\\
     85 & TML & Temporal Mid L &  86 & TMR & Temporal Mid R\\
     87 & TPML & Temporal Pole Mid L &  88 & TPMR & Temporal Pole Mid R\\
     89 & TIL & Temporal Inf L &  90 & TIR & Temporal Inf R\\
    \bottomrule
    \end{tabular}}
\end{table}

\end{document}